%% file: main.tex
\documentclass[sigconf]{acmart}

\AtBeginDocument{%
	\providecommand\BibTeX{{%
			\normalfont B\kern-0.5em{\scshape i\kern-0.25em b}\kern-0.8em\TeX}}}

\usepackage{graphicx}
\usepackage{textcomp}

\usepackage{xcolor}
\usepackage{natbib}
\usepackage[utf8]{inputenc} 
\usepackage[T1]{fontenc}    
\usepackage{hyperref}       
\usepackage{url}            
\usepackage{booktabs}       
\usepackage{nicefrac}       
\usepackage{microtype}      
\usepackage{verbatim}
\usepackage[ruled,linesnumbered,noend]{algorithm2e} 
\usepackage{amsmath,amsthm}
\usepackage{subfigure}
\usepackage{caption}
\usepackage{subfigure}
\usepackage{bm}
\usepackage{verbatim}
\usepackage{hyperref}
\newtheorem{theorem}{Theorem}
\newtheorem{lemma}{Lemma}
\newtheorem{remark}{Remark}
\usepackage{float}
\usepackage{wrapfig}
\usepackage{soul}
\allowdisplaybreaks
\input{symbols_commands}
\newcommand{\algname}{\textsf{SYNTHESIS }}
\newcommand{\algnamens}{\textsf{SYNTHESIS}}




\copyrightyear{2022} 
\acmYear{2022} 
\setcopyright{acmcopyright}\acmConference[MobiHoc '22]{The Twenty-third International Symposium on Theory, Algorithmic Foundations, and Protocol Design for Mobile Networks and Mobile Computing}{October 17--20, 2022}{Seoul, Republic of Korea}
\acmBooktitle{The Twenty-third International Symposium on Theory, Algorithmic Foundations, and Protocol Design for Mobile Networks and Mobile Computing (MobiHoc '22), October 17--20, 2022, Seoul, Republic of Korea}
\acmPrice{15.00}
\acmDOI{10.1145/3492866.3549722}
\acmISBN{978-1-4503-9165-8/22/10}

\begin{document}

\title[ SYNTHESIS for Distributed Learning in Computing Clusters]{SYNTHESIS: A Semi-Asynchronous Path-Integrated Stochastic \!\!\!Gradient Method for Distributed Learning in Computing Clusters}

%
%
%
%

\author{Zhuqing Liu$^1$, Xin Zhang$^{2}$, and Jia Liu$^{1}$
}
\affiliation{
	\institution{$^1$Department of Electrical and Computer Engineering, The Ohio State University}
	\institution{$^2$Department of Statistics, Iowa State University}		
	\country{}
}

\begin{abstract}
	To increase the training speed of distributed learning, recent years have witnessed a significant amount of interest in developing both synchronous and asynchronous distributed stochastic variance-reduced optimization methods.
	However, all existing synchronous and asynchronous distributed training algorithms suffer from various limitations in either convergence speed or implementation complexity.
	This motivates us to propose an algorithm called \algname (\ul{s}emi-as\ul{yn}chronous pa\ul{th}-int\ul{e}grated \ul{s}tochastic grad\ul{i}ent \ul{s}earch), which leverages the special structure of the variance-reduction framework to overcome the limitations of both synchronous and asynchronous distributed learning algorithms, while retaining their salient features.
	We consider two implementations of \algname under distributed and shared memory architectures. 
	We show that our \algname algorithms have \(O(\sqrt{N}\epsilon^{-2}(\Delta+1)+N)\) and \(O(\sqrt{N}\epsilon^{-2}(\Delta+1) d+N)\) computational complexities for achieving an \(\epsilon\)-stationary point in non-convex learning under distributed and shared memory architectures, respectively, where \(N\) denotes the total number of training samples and \(\Delta\) represents the maximum delay of the workers. 
	Moreover, we investigate the generalization performance of \algname by establishing algorithmic stability bounds for quadratic strongly convex and non-convex optimization. 
	We further conduct extensive numerical experiments to verify our theoretical findings.

\end{abstract}

\maketitle

\section{Introduction}\label{sec:intro}

From the early days of machine learning (ML), the classical first-order stochastic gradient descent (SGD) method has been used as the workhorse training algorithm due to 
the high dimensionality of ML computing tasks and the large size of datasets.
However, it is well known that the SGD algorithm suffers from slow convergence.
To accelerate the traditional SGD approach, there have been two main approaches in the literature.
The first approach is to exploit algorithmic techniques, including momentum~\cite{polyak1964Momentum}, adaptive learning rates~\cite{george2006adaptive}, etc.
One of the most notable algorithmic acceleration approaches in recent years is the family of {\em ``variance-reduced''} (VR) methods (see, e.g., SVRG~\cite{johnson2013svrg}, SAG~\cite{schmidt2017SAG}, and SAGA~\cite{defazio2014saga} and many of their variants).
The basic idea of these VR methods is to construct accurate gradient estimators periodically by recomputing (near) full gradients 
to reduce variance.  
In the VR family, the state-of-the-art is the SPIDER method (stochastic path-integrated differential estimator) by \cite{fang2018spider} and its enhanced version called SpiderBoost~\cite{wang2018spiderboost}
(see Section~\ref{sec:related} for more in-depth discussions). 
The second major approach to accelerate SGD is to leverage the {\em parallelism} in distributed computing clusters (e.g., from chip-scale to datacenter-scale GPU farms), thanks to SGD's decomposable structure implied by the use of mini-batches. 
Notably, both {\em distributed memory} parallelism on multiple GPUs \cite{agarwal2011distributed, lian2015asynchronous, zhang2014asynchronous} and {\em shared memory} parallelism on a multi-core machine \cite{recht2011hogwild, zhao2016fast} have been exploited in SGD-based ML tasks, such as parallel SVM \cite{tavara2019svm}, parallel matrix factorization \cite{tang2015providing, yao2018scalable} and distributed deep learning \cite{wen2017terngrad}, just to name a few.

However, in the distributed and parallel computing approach for distributed ML, a key design dilemma is the architectural choices between {\em ``synchronous'' and ``asynchronous''} implementations for the distributed SGD algorithm.
A salient feature of synchronous parallel algorithms is that they have a more stable convergence performance in general.
However, synchronous implementations suffer limitations in complexity in maintaining a common clock, straggling problems, and periodic spikes in data traffic.
In comparison,
asynchronous algorithms are easier to implement and cause less network traffic congestions and delays.
However, a major limitation of asynchronous parallel algorithms is the inevitable impact of {\em stale stochastic gradient} information due to asynchronous updates.
If not treated appropriately, the stale stochastic gradient information could significantly degrade the convergence performance of the asynchronous algorithm.

The pros and cons of ``synchronous vs. asynchronous parallelisms'' in the parallel SGD implementation motivate us to pursue a new {\bf semi-asynchronous} distributed optimization method that achieves the best of both worlds while avoiding their pitfalls.
Interestingly, our {\em key idea} in addressing the problems in ``synchronous vs. asynchronous parallelisms'' (i.e., the second SGD acceleration approach) comes from the VR-based algorithmic acceleration (i.e., the first SGD acceleration approach).
%
%
%
Specifically, we show that the ``{\em double-loop}'' structure of the VR-based algorithms~\cite{johnson2013svrg,fang2018spider,wang2018spiderboost} naturally implies an elegant {\em semi-asynchronous} implementation, which entails a simple implementation as in asynchronous distributed algorithms, while providing strong convergence and generalization performance guarantees as in synchronous distributed algorithms.
%
%
%
%
This key insight enables us to develop a new distributed learning algorithm called \algname (\ul{s}emi-as\ul{yn}chronous pa\ul{th}-int\ul{e}grated \ul{s}tochastic grad\ul{i}ent \ul{s}earch).
Our main results and contributions are summarized as follows:

\begin{list}{\labelitemi}{\leftmargin=1em \itemindent=-0.09em \itemsep=.2em}
	\item We first analyze the convergence of \algname for non-convex  optimization for learning problems under the distributed memory. 
	We show that \algname achieves a computational  complexity $O(\sqrt{N}\epsilon^{-2}(\Delta+1)+N)$ in terms of stochastic first-order oracle (SFO) evaluations to find an $\epsilon$-approximate first-order stationary point (to be defined later), where $N$ is the number of training samples, and $\Delta$ denotes the maximum delay in asynchronous update across all workers. 
	Our result shows that delays in asynchrony only linearly affects the hidden constant in convergence performance and does {\em not} slow down the convergence rate order.
	Also, fewer number of training samples $N$ may speed up the running time when they reach similar training loss results (i.e., lower sample complexity). 
	
	\item Next, we analyze the convergence of \algname under the shared memory architecture.
	We show that the \algname method achieves a computational complexity $O(\sqrt{N}\epsilon^{-2}(\Delta+1) d+N)$ in terms of SFO evaluations. 
	This result reveals an interesting insight that, under shared memory, one needs to pay an additional cost that is $d$ times larger due to the restriction that only one vector coordinate can be updated at a time.
	Nonetheless, this restriction does {\em not} affect the convergence rate with respect to $\epsilon$.

	\item We further study the generalization performance of \algname under both distributed memory and shared memory architectures.
	We establish the upper bounds of the expected generalization errors of \algname for both convex and non-convex learnings.  
	Our generalization bounds show that larger step-sizes and smaller training datasets cause larger generalization errors in \algnamens. 
	These results also characterize a fundamental trade-off between convergence and generalization in distributed learning algorithms. 
\end{list}


The rest of the paper is organized as follows. 
In Section~\ref{sec:related}, we first discuss related work and preliminaries of distributed memory and shared memory architectures.
In Section~\ref{sec:model}, we present the system model, problem formulation, and basic assumptions. 
In Sections~{\large \ref{sec:distr_mem}} and~\ref{sec:shared_mem}, we analyze convergence and generalization performances of the \algname method under the distributed memory and shared memory parallelisms, respectively.
Section~\ref{sec:numerical} presents numerical results and Section~\ref{sec:conclusion} concludes this paper. 
We note that the full version of the proofs in \cite{proof}.

\section{Related Work and Preliminaries}\label{sec:related}

In this section, we first provide an overview on variance-reduced first-order methods and asynchronous distributed first-order methods.
Then, we offer some necessary background on two widely adopted parallel computing architectures, namely distributed memory and shared memory.
To define the stochastic first-order oracle (SFO) complexity, we first introduce the notion of first-order $\epsilon$-stationary point.
We say that a point $\x$ is a first-order $\epsilon$-stationary point of a function $f(\cdot)$ if the first-order gradient condition is satisfied: $\mathbb{E} [ \| \nabla f(\x) \|^2 ] \leq \epsilon^2$. 


\smallskip
\textbf{ $\bullet$ Variance-Reduced First-Order Methods:} 
The  history of the basic SGD algorithm dates back to 1951~\cite{robbins1951stochastic}. 
As mentioned in Section~\ref{sec:intro}, it is well-known that the basic SGD method achieves an $\epsilon$-approximate stationary point with an SFO  evaluation cost of $O(\epsilon^{-4})$ \cite{ghadimi2013stochastic}. 
To address this limitation, variance reduction techniques (VR) \cite{johnson2013accelerating, roux2012stochastic} have emerged in recent years to improve the convergence rate of SGD. 
Notably, SAGA~\cite{defazio2014saga}  and SVRG~\cite{johnson2013svrg} achieve an SFO complexity of $O(N^{2/3}\epsilon^{-2})$ iterations to attain a first-order $\epsilon$-stationary point for non-convex optimization. 
In the VR family, SPIDER~\cite{fang2018spider} and SpiderBoost~\cite{wang2018spiderboost} are the state-of-the-art, achieving an SFO complexity of $O(\min({\sqrt{N}\epsilon^{-2}},\epsilon^{-3}))$.

\smallskip
\textbf{ $\bullet$ Asynchronous Distributed Optimization for Learning:} 
One of the earliest asynchronous SGD-type algorithms is HOGWILD!~\cite{recht2011hogwild},  which is a lock-free asynchronous parallel implementation of SGD under the shared memory architecture with a sublinear convergence rate for strongly convex problems. 
Roughly the same time, the convergence of SGD with asynchronous gradients was studied in \cite{agarwal2011distributed}, which has an SFO complexity of $O(\epsilon^{-4})$ if the random gradient update delay is i.i.d second-moment-bounded. 
This $O(\epsilon^{-4})$ SFO complexity is the same as that of the synchronous and non-delayed case. 
Later, a coordinate-descent-based asynchronous parallel optimization algorithm termed ARock was proposed in \cite{peng2016arock}, which generalizes asynchronous SGD (Async-SGD) for solving convex problems. 
In \cite{huo2016asynchronous,reddi2015variance} an asynchronous stochastic VR-based Async-SVRG method is proposed to achieve linear convergence for bounded delays for convex problems. 
However, the computation complexity of Async-SVRG suffers from $O({n^{-2/3}\epsilon^-2})$, which is higher compared to our \algname algorithm from $O(n^{-1/2}\epsilon^{-2})$. 
Our better sample complexity result is due to the recursive structure in the asynchronous gradient estimator in our \algname algorithm.

\begin{figure}[t!]
    \begin{minipage}[t]{0.48\linewidth}
        \centering
        \includegraphics[width=1\textwidth]{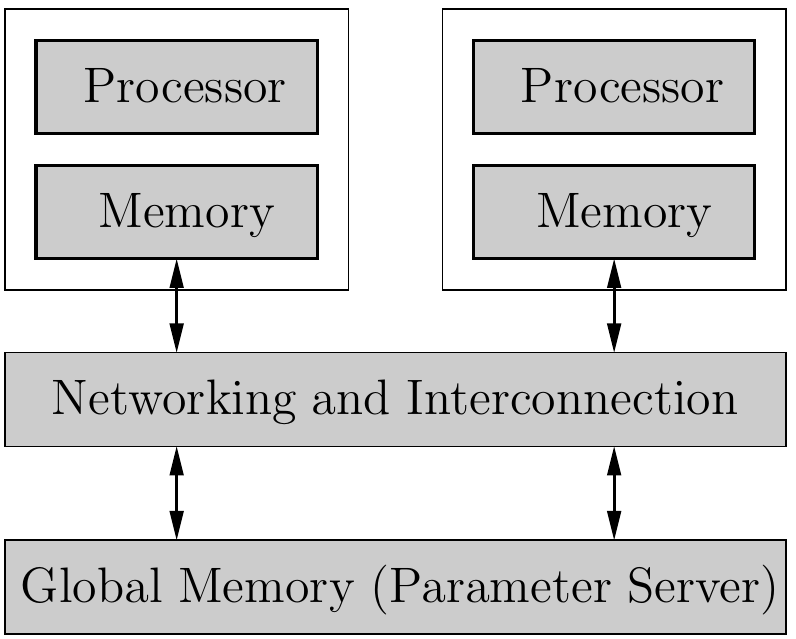}
        \caption{The distributed memory architecture.} \label{fig_distr_mem}
    \end{minipage}%
    \hspace{0.02\linewidth}
    \begin{minipage}[t]{0.48\linewidth}
        \centering
        \includegraphics[width=1\textwidth]{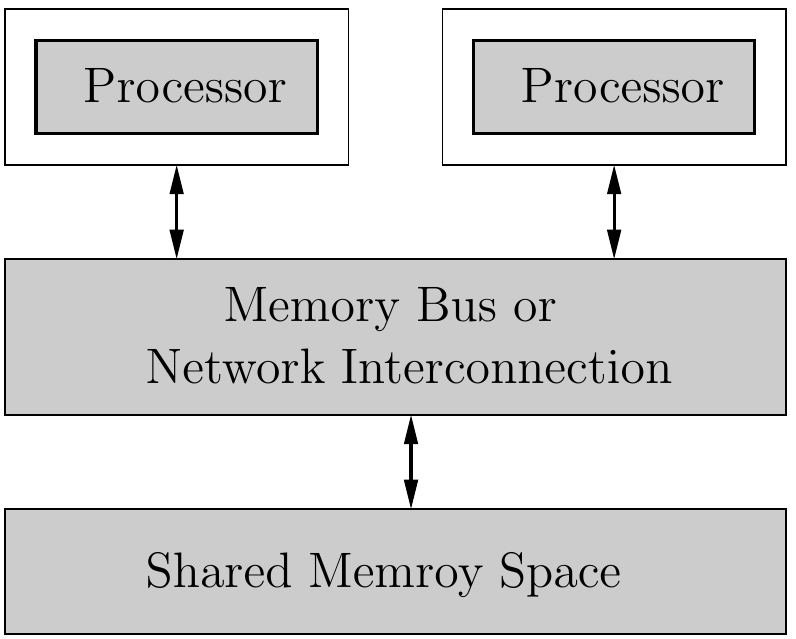}
        \caption{The shared memory architecture.}\label{fig_shared_mem}
    \end{minipage}%
    \vspace{-.2in}
\end{figure}

\smallskip
\textbf{ $\bullet$ Distributed Memory vs. Shared Memory:} 
To facilitate later discussions, we provide a brief overview on distributed memory and shared memory parallel architectures, two of the most common distributed computing architectures in practice.

\smallskip
{\em 1) Distributed Memory:}
As it name suggests, distributed memory refers to a multi-processor computing system, where each processor has its own private memory. Computational tasks can only operate on local data, and if remote data is required, the computational task must communicate with one or more remote processors. 
As illustrated in Fig.~\ref{fig_distr_mem}, a distributed memory system typically contains processors and their associated local memory, as well as some form of interconnection that allows programs on each processor to interact with each other. 
In the context of distributed ML, thanks to the independent memory, each computing unit can compute a stochastic gradient of the objective function based on local data and update {\em all coordinates} at the parameter server {\em simultaneously} without affecting other workers' operations.

\smallskip
{\em 2) Shared Memory:}
In contrast, a shared memory system offers a {\em single memory space} that can be simultaneously accessed by all processors/programs. 
Depending on context, programs may run on a single processor with multiple threads or multiple processors. 
In the context of distributed ML, the shared memory architecture is often used by a single machine with multiple cores/GPUs. 
The parameter values are stored in the shared memory and multiple threads can access them. 
Each thread reads the parameter values from the shared memory and randomly chooses a batch of samples to compute a stochastic gradient. 
Each thread then updates the current parameters with its stochastic gradient. 
However, due to the shared memory restriction, each thread can only read or write {\em a single coordinate at a time} to prevent race conditions~\cite{recht2011hogwild}.

\smallskip
With the basic notions of distributed and shared memory architectures above, we are now in a position to present our proposed \algname algorithm in Section~\ref{sec:model}.

\section{System Model, Problem Formulations and Basic Assumptions} \label{sec:model}
%

In this section, we present the system model, problem formulation and the assumptions used in this paper. 
We consider a distributed learning system with $P$ workers and a parameter server.
There are $N$ data samples in total in the global dataset, which is denoted as $\mathcal{S} = (\xi_1 ,..., \xi_N)$.
Each sample $\xi_i$ is independently and identically distributed (i.i.d.) following a latent distribution $\mathcal{D}$.
The global dataset $\mathcal{S}$ in dispersed in each worker, and each local dataset at worker $p$ is denoted as $\mathcal{S}_{p}$, with $\sum_{p=1}^{P} |\mathcal{S}_{p}| = N$.
For simplicity, we assume equal distribution and let $n \triangleq |\mathcal{S}_{p}| = N/P$.\footnote{For simplicity, we assume here that $N$ is divisible by $P$
Note that, with slightly more cumbersome notation in the analysis, all our proofs and results continue to hold in cases where $N$ is not divisible by $P$ or unequal distributions.
}
The goal of the distributed learning system is to solve an optimization problem, which is typically non-convex and in the following form: 
\begin{align*}
\min_{\x\in \mathbb{R}^d} f(\x) \triangleq  \frac{1}{N} \sum_{p=1}^{P} \sum_{i=1}^{|\mathcal{S}_{p}|} f_i(\x,\xi_i)=\frac{1}{nP} \sum_{p=1}^{P} \sum_{i=1}^{n} f_i(\x,\xi_i).
\end{align*} 
In this paper, we make the following assumptions:
\begin{assum}[Bounded Objective Function] \label{bnded_f}
	The function $f(\cdot,\cdot)$ is bounded from below, i.e., $f^* = \mathrm{inf}_{\x \in \mathbb{R}^d} f(\x,\cdot)>- \infty $.
\end{assum}
	
\begin{assum}[Continuously Differentiable Loss Function] \label{assum_diff}
	The loss function $f(\cdot,\cdot)$ is continuously differentiable.
\end{assum}	

\begin{assum}[{$M$-Lipschitz Loss Function}:] \label{assum_Lip_loss}
	$f(\cdot,\cdot)$ is $M$-Lipschitz continuous, i.e., there exists a constant $M > 0$ such that $| f(\u,\xi_i) - f(\v,\xi_i)| \leq M \|\u-\v\|$, $\forall \u,\v \in \mathbb{R}^d$, $\forall \xi_i$.
\end{assum}	

\begin{assum}[Uniformly Bounded-Size Gradient:] \label{assum_bnded_grad}
There exists a constant $M>0$ such that $\sup_{\mathcal{S}}\|\nabla f(\cdot,\xi_i)\|\leq M$. 
\end{assum}

\begin{assum}[$L$-Lipschitz Smoothness] \label{assum_Lip_grad}
The function $f(\cdot,\cdot)$ is $L$-Lipschitz smooth, i.e., there exists a constant $L>0$ such that $\norm{\nabla f(\u,\xi_i) - \nabla f(\v,\xi_i) } \leq L \norm{ \u-\v }$, $\forall \u,\v \in \mathbb{R}^d$, $\forall \xi_i$.
\end{assum}	

\begin{assum}[Bounded Delay in Asychrony] \label{assum_bnded_delay}
	The maximum of random delay $\tau$ of asynchronous stochastic gradient updates is upper bounded by a constant $\Delta >0$, i.e., $\tau \leq \Delta$.
\end{assum}

Several remarks on Assumptions~\ref{bnded_f}--\ref{assum_bnded_delay} are in order.
Assumptions~\ref{bnded_f}--\ref{assum_Lip_grad} are standard in convergence analysis in the  literature.
Assumption~\ref{assum_bnded_delay} is also a standard assumption in the asynchronous computing literature and holds in most practical computing systems.
We note that Assumptions~\ref{assum_Lip_loss} and~\ref{assum_bnded_grad} share the same constant $M$ since Assumption~\ref{assum_bnded_grad} is implied by Assumption~\ref{assum_Lip_loss} (but we state Assumption~\ref{assum_bnded_grad} explicitly for convenience in subsequent analysis).
The bounded gradient norm in Assumption~\ref{assum_bnded_grad} is often guaranteed by stability-inducing operations, e.g., regularization, projection, or gradient clipping.

\section{\algnamens: Distributed Memory} \label{sec:distr_mem}

In this section, We first propose the \algname algorithm for the distributed memory architecture to handle distributed ML problems with a large dataset that cannot fit in a single machine's storage.
We first describe our algorithm in Section~\ref{subsec:sa_spiderboost_dm} and then present its convergence analysis in Section~\ref{subsec:sa_spiderboost_dm_convergence}.
Lastly, we will analyze the generalization performance of \algname in Section~\ref{subsec:sa_spiderboost_dm_gen}.

\subsection{Algorithm Description} \label{subsec:sa_spiderboost_dm}

As mentioned in Section~\ref{sec:intro}, due to the various pros and cons in ``synchronous vs. asynchronous parallelisms,''
we pursue a new {\em semi-asynchronous} approach to achieve the best of both worlds in this paper.
Our {\em key idea} is motivated by the observation that the double-loop structure of the state-of-the-art VR-based methods~\cite{johnson2013svrg,fang2018spider,wang2018spiderboost} can be leveraged to construct a simple and elegant {\em semi-asynchronous algorithm}.
The server and worker algorithms of our \algname are presented in Algorithms~\ref{alg:sa_sb_dm_server} and~\ref{alg:sa_sb_dm_worker}.
In what follows, we take a closer look at these two algorithms.

\smallskip
{\bf 1) The Inner Loop of Algorithm~\ref{alg:sa_sb_dm_server} (Asynchronous Mode):}
From the Server Code in Algorithm~\ref{alg:sa_sb_dm_server}, we can see that the inner loop (Lines~8--9) is executed $q-1$ iterations in total.
In the inner loop of the \algname algorithm, on the worker side (Steps~3--6 of the Worker Code in Algorithm~\ref{alg:sa_sb_dm_worker}), each worker $i$ independently retrieves the freshest parameter $\x_{new}$ from the parameter server and then randomly selects a mini-batch $S$ of data samples from its local dataset and computes a stochastic gradient. 
Also, $\v_{old}$ is an unbiased gradient estimate of $\nabla f(\x_{old})$.  
Each worker immediately reports the computed stochastic gradient $\v_{new}$ to the parameter server. 
The parameter server then updates its current parameters with this stochastic gradient information {\em without waiting} for other workers, hence operating in an {\em asynchronous} mode. 
Note that, while this worker is trying to send its gradient information to the parameter server, the parameter server may have been updated by other workers with gradient information associated with a fresher $\x$-value. 
Thus, this particular worker's gradient information could turn out to be “delayed.” 

\smallskip
{\bf 2) The Outer Loop of Algorithm~\ref{alg:sa_sb_dm_server} (Synchronous Mode):}
On the other hand, at the beginning of every timeframe of length $q$, the parameter server executes the outer loop (Line~1--7 in Algorithm~\ref{alg:sa_sb_dm_server}).
Specifically, the parameter server sends out an interruption signal to all workers to cancel the unfinished computation at each worker if there is any (Step~3 in Algorithm~\ref{alg:sa_sb_dm_server}).
Upon receiving the interruption signal, every worker stops and then retrieves the freshest parameter $\x_k$ from the parameter server and uses {\em all} samples in its local dataset to compute a local full gradient and send the result to the parameter server (Step~2 in Worker Code in Algorithm~\ref{alg:sa_sb_dm_worker}).
The parameter server collects the gradient information from {\em all} workers, hence operating in a {\em synchronous} mode.
Finally, the server computes the global full gradient (Step~6 in Algorithm~\ref{alg:sa_sb_dm_server}) and sends out the most recent $\x$- and $\v$-values as $\x_{old}$ and $\v_{old}$ to all workers (Step~7 in Algorithm~\ref{alg:sa_sb_dm_server}), which are needed for the first variance-reduced operation in the first iteration of the inner loop (see Step~2 in Algorithm~\ref{alg:sa_sb_dm_worker}).
Finally, the parameter server updates the $\x$-value along the global full gradient direction (Step~10 in Algorithm~\ref{alg:sa_sb_dm_server}).

\begin{algorithm}[t!]
	\SetAlgoLined
	\SetNoFillComment
	\KwIn{$ q, K \in \mathbb{N}$}
	\For{$k=0,1,...,K-1$}{
		\If{$mod(k,q)=0$}{ 
			Send a signal to all workers to interrupt all unfinished computing jobs at each worker. \\
			Push $\x_k$ to all workers.\\
			Wait until receiving $G_k^{(p)}$, $\forall p$, from all workers.\\
			Compute full gradient $\v_k=\nabla f(\x_k)= \frac{1}{N} \sum_{p=1}^{P} G_k^{(p)}$\\
			Broadcast $\x_{old}=\x_k$ and $\v_{old} = \v_k$ to all workers.\\}
		\Else {
			Let $\v_k=\v_{new}$ be the feedback from a specific worker with delays. 		}
		Update parameter $\x_{k+1}=\x_{k}- \eta \v_k$.\\
	}
	\KwOut{$\x_{\zeta}$, where the index $\zeta$ is chosen uniformly at random from $\{1,\ldots,K\}$.}
	\caption{The Parameter Server Code of the Distributed \algname Algorithm.}
	\label{alg:sa_sb_dm_server}
\end{algorithm}

\begin{algorithm}[t!]
	\SetAlgoLined
	\SetNoFillComment
	If receiving interrupt signal at any time, go immediately to Step 2; otherwise, go to Step~3.
	
	Receive $\x_k$ form server. Compute $G_k^{(p)}=  \sum_{i\in\mathcal{S}_{p}}\nabla f(\x_{k}, \xi_i)$ and send it to the parameter server. Wait and receive $\x_{old}$ and $\v_{old}$ from the parameter server and go to Step~1.     \quad	{\em /* In the following steps, stop and go to Step~2 immediately upon receiving an interrupt signal from the server */}
	
	Pull the most fresh parameter of $\x$ as $\x_{new}$ from server. 
	
	Randomly select a subset $S$ of samples from $\mathcal{S}_p$. 
	
	Compute $\v_{new} =\frac{1}{|S|} \sum_{i\in S} (\nabla f(\x_{new}, \xi_i)-\nabla f(\x_{old}, \xi_i) + \v _{old})$ and send it to the parameter server, which takes combined $\tau_k$ time-slots in computation and communication. 
	
	Let $\x_{old}=\x_{new},  \v_{old}=\v_{new}$, go to Step~3. 

	\caption{The Worker Code of Each Worker $p$ of the Distributed \algname Algorithm.}
	\label{alg:sa_sb_dm_worker}
\end{algorithm}

\begin{figure}[t!]
	\centering
		\includegraphics[width=1\linewidth]{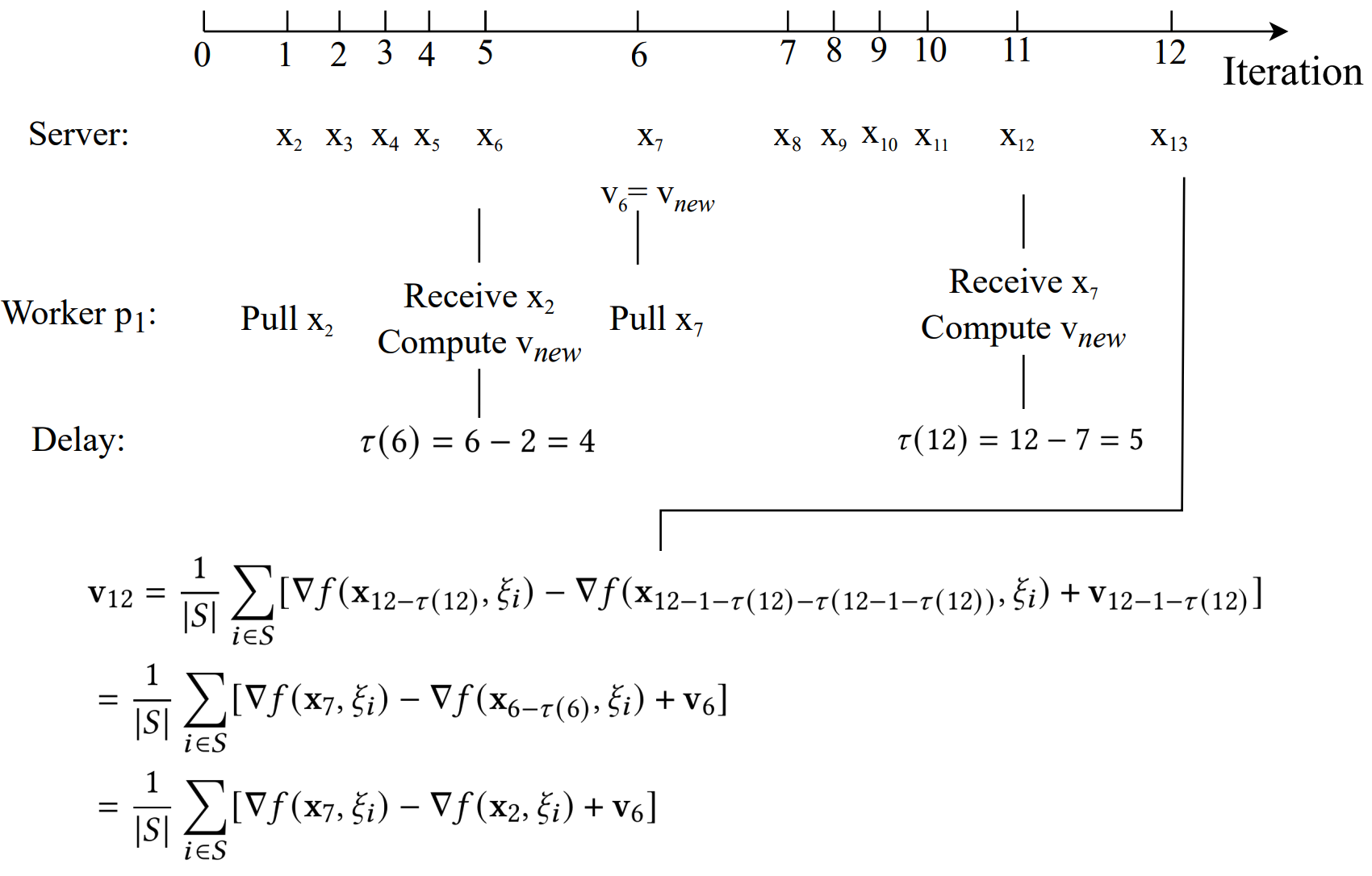} 
		\vspace{-.25in}
	\caption{An example of the stale stochastic gradients due to asynchrony in \algnamens: In the computation of direction $\v_{12}$ (for $\x_{13}$), stale stochastic gradients corresponding to $\x_{7}$ and $\x_{2}$ are used instead of $\x_{12}$ and $\x_{11}$.}
		\label{iteration} 
		\vspace{-.2in}
\end{figure}

In each iteration, parameter $\x$ is updated through the following update rule:
$
\x_{k+1}=\x_{k}- \eta \v_k
$,
where $\eta$ is a constant learning rate, $\v_k$ represents the {\em variance-reduced} update direction in iteration $k$.
Following from 
Algorithms~\ref{alg:sa_sb_dm_server} (Step~10) and \ref{alg:sa_sb_dm_worker} (Step~5), we can express the algorithmic update of \algname as follows:
\begin{align} \label{eqn_vk}
&\v_k=\v_{new} =\frac{1}{|S|} \sum_{i\in S} (\nabla f(\x_{new}, \xi_i)-\nabla f(\x_{old}, \xi_i) + \v _{old})\notag\\
&=\frac{1}{|S|} \sum_{i\in S} [\nabla f(\x_{k-\tau(k)},\xi_i)- \nabla f(\x_{k-1-\tau(k)-\tau(k-1-\tau(k))},\xi_i) \nonumber\\
&\quad \quad +\v_{k-1-\tau(k)}],
\end{align} 
where the index $i$ denotes the index of a sample, 
and $\tau(k)$ denotes the delay of stochastic gradient information used to update $\x_k$ in iteration $k$ (satisfying $\tau(k) \leq \Delta$ for any iteration $k$).
Thus, $\x_{k-\tau(k)}$  denotes the ``new'' parameter a worker uses to compute the gradient with delay $\tau(k)$ in a worker, and  
$\x_{(k-1-\tau(k))-\tau(k-1-\tau(k))}$ denotes the ``old'' parameter we used to compute the gradient with additional delay $\tau(k-1-\tau(k))$ (cf. Algorithm~\ref{alg:sa_sb_dm_worker}).  

One important remark regarding the update in \eqref{eqn_vk} is in order. 
We note that the algorithmic update in (\ref{eqn_vk}) integrates the path of the $\{\v_k\}$ trajectory (hence the name \algnamens), which shares some similarity with the class of reccurssive variance-reduced gradient estimators in the family of VR methods (e.g., SARAH~\cite{nguyen2017sarah}, SPIDER~\cite{fang2018spider}, SpiderBoost~\cite{wang2018spiderboost}, and PAGE~\cite{li2021page}).
Thus, it is insightful to compare Eq.~\eqref{eqn_vk} with these existing works that have the following recursive form (denoted as ``rec'') in the sequence $\{\v_k^{\mathrm{(rec)}}\}$:
\begin{align} \label{eqn_vk_existing}
\v_k^{\mathrm{(rec)}}=\frac{1}{|S|} \sum_{i\in S} [\nabla f(\x_{k},\xi_i)- \nabla f(\x_{k-1},\xi_i)+\v_{k-1}^{\mathrm{(rec)}}].
\end{align} 
The key difference and novelty in \algname compared to these existing recursive VR methods is the {\em presence of stale stochastic gradient information} (due to asynchrony) from the iteration $k-\tau(k)$ and its previous step in iteration $k-1-\tau(k-1)$ (see Fig.~\ref{iteration} for an example of stale stochastic gradients in \algnamens).
This asynchrony-induced staleness coupled with the recursive structure render the algorithm performance analysis significantly more complex and challenging than the aforementioned existing works~\cite{nguyen2017sarah, fang2018spider, wang2018spiderboost, li2021page}.

Upon receiving the updated gradient from a specific worker, the parameter server updates its current parameter with $\v_k$ if in the inner loop, or wait until finishing the collection of all workers' $\v$-values to compute the full gradient and update. 

\subsection{Convergence Performance Analysis} \label{subsec:sa_spiderboost_dm_convergence}

In this subsection, we first present the main convergence result of \algname under the distributed memory architecture.
Due to space limitation, we provide a proof sketch here and relegate the proof to the full version of this paper. 

\begin{theorem}[Convergence of \algname for Distributed Memory] \label{aaa}
Let $f^{*}$ denote the global optimal value.
Under distributed memory and assumptions \ref{bnded_f}-\ref{assum_bnded_delay}, for some $\epsilon >0, P\leq \sqrt{N}$, by choosing parameters $ q=|S|=\sqrt{N}$ and $\eta\leq \frac{1}{4L(\Delta+1)}$ for a given maximum delay $\Delta>0$, the \algname algorithm outputs an $\x_\zeta$ that satisfies $\mathbb{E} [ \| \nabla f(\x_\zeta) \|^2 ]\leq\epsilon^2$ if the total number of iterations $K$ satisfies 
$K = \mathcal{O} \big(  \frac{ f(\x_0)-f^* }{\epsilon^2} \big)$.
This also implies that the total SFO complexity is $\mathcal{O} (\sqrt{N}\epsilon^{-2}(\Delta+1)+N)$. 
\end{theorem}

\begin{remark}{\em
Two remarks on Theorem~\ref{aaa} are in order:
i) Theorem~\ref{aaa} shows that the output of \algname achieves a first-order stationary point with total computational complexity $O(\sqrt{N}\epsilon^{-2}(\Delta+1)+N)$. 
For the special case with maximum delay $\Delta=0$ , the result of Theorem~\ref{aaa} recovers the total computational complexity $O(\sqrt{N}\epsilon^{-2}+N)$ of the state-of-the-art synchronous VR-based algorithms~\cite{fang2018spider,wang2018spiderboost}, where $N$ is the total training samples.
ii) We note that the analysis of \algname is very different from those of the synchronous VR-based algorithms. 
From a theoretical perspective, the major challenge in the proof of Theorem~ \ref{aaa} is
the asynchrony between different workers.
}
\end{remark}

\begin{proof}[Proof Sketch of Theorem~\ref{aaa}]
Here, we provide a proof sketch due to space limitation. 
To prove the stated result in Theorem~\ref{aaa}, we start with evaluating $\mathbb{E}[ \| \nabla f(\x_\zeta) \|^2 ]$ ($\zeta$ is chosen uniformly at random from $\{1,\ldots,K\}$).
Following the inequality of arithmetic and geometric means, we have:
\begin{align} \label{eqn_step1}
&\mathbb{E} [ \| \nabla f(\x_\zeta) \|^2 ]
=\mathbb{E} [ \| \nabla f(\x_\zeta)-\v_\zeta+\v_\zeta \|^2 ] \notag\\&
\leq  2\mathbb{E} [ \| \nabla f(\x_\zeta)-\v_\zeta \|^2 ]+2\mathbb{E} [ \| \v_\zeta \|^2 ].
\end{align}                                                             
Next, we bound the terms $\mathbb{E} [ \| \nabla f(\x_\zeta) - \v_\zeta \|^2 ]$ and $\mathbb{E}[ \| \v_\zeta \|^2 ]$ on the right-hand-side of (\ref{eqn_step1}) individually.

{\em Step 1) Bounding the gradient estimator bias $\mathbb{E} [ \| \v_\zeta - \nabla f(\x_\zeta) \|^2 ]$:}
Toward this end, we first bound the distance between the local update direction and the node-average gradient direction $\mathbb{E} [ \| \v_{k}-\frac{1}{N}\sum_{i=1}^{N}\nabla f(\x_{k-\tau(k)},\xi_i) \|^2 ]$ of the inner loop. 
Let $n_k=\lceil k/q \rceil$ denote the epoch index that iteration $k$ belongs to (i.e., $(n_k-1)q\leq k \leq n_kq -1$).
Then, from the inner loop operations, we can show the following relationship:
\begin{align*}
&\mathbb{E} \bigg[ \bigg\| \v_{k}-\frac{1}{N}\sum_{i=1}^{N}\nabla f(\x_{k-\tau(k)},\xi_i) \bigg\|^2 \bigg]
\leq  \frac{L^2\eta^2(\Delta+1) }{|S|} \cdot\\&
\sum_{j=(n_k-1)q}^{k-1-\tau(k)} \mathbb{E} [ \| \v_{j} \|^2 ] + \mathbb{E} \bigg[ \bigg\|\v_{(n_k-1)q}-\frac{1}{N}\sum_{i=1}^{N}\nabla f(\x_{(n_k-1)q },\xi_i) \bigg\|^2 \bigg].
\end{align*}
By using $\beta_{1}$ defined later in Step~2) and the bound on $\sum_{i=1}^{K-1} \mathbb{E} [ \|\v_{i}\|^{2} ]$, we can further bound $\mathbb{E} \big[ \big \| \v_\zeta-\nabla f(\x_\zeta) \big\|^2 \big]$ as:
\begin{multline}
\mathbb{E} \left[ \left \| \v_\zeta-\nabla f(\x_\zeta) \right\|^2 \right] \leq\left[\frac{{2L^2}(\Delta^2+\frac{q(\Delta+1)}{|S|})\eta^3}{\beta_1}+2\right]\epsilon_1^2+ \\
 {2L^2}(\Delta^2+\frac{q(\Delta+1)}{|S|})\eta^2\left[\frac{f(\x_0)-f^*}{K\beta_1}\right].
\end{multline}                                                                               

\smallskip
{\em Step 2) Bounding  the second moment of the moving direction $\mathbb{E} [ \| \v_\zeta \|^2]$:}
Consider the term $\mathbb{E} [ \| \v_\zeta \|^2]$ in (\ref{eqn_step1}). To evaluate $\mathbb{E} [ \| \v_\zeta \|^2 ]$, we start from the iteration relationship of our \algname algorithm.
From the $L$-smooth assumption, it can be shown that:
\begin{align} \label{eqn_step2}
\!\!	f(\x_{k+1}) \!\leq\! f(\x_k) \!+\! \frac{\eta}{2} \big \| \v_k \!-\! \nabla f(\x_k) \big\|^2 \!-\!(\frac{\eta}{2} \!-\! \frac{L\eta^2}{2})\big \| \v_k \big\|^2. \!\!
\end{align}
It follows from (\ref{eqn_step2}) and inductively bounding $\mathbb{E}[ f(\x_{k+1})]$ in the inner loop  $(n_k-1)q\leq k \leq n_kq -1$ that:
\begin{multline} \label{eqn_step3}
	\mathbb{E}[f(\x_{k+1})] \leq \mathbb{E}[f(\x_{(n_k-1)q})] + \!\!\!\! \sum_{i=(n_k-1)q}^k \eta\epsilon_1^2- 
	\big[ \frac{\eta}{2}-\frac{L\eta^2}{2} \\- L^2\eta^3\big(\frac{q(\Delta+1) }{|S|}+\Delta^2 \big) \big] \sum_{i=(n_k-1)q}^k\mathbb{E} \big[ \big \| \v_{i} \big\|^2 \big].
\end{multline}                                                             
Next, letting $\beta_{1} \triangleq \big[ \frac{\eta}{2}-\frac{L\eta^2}{2}- L^2\eta^3\big(\frac{q(\Delta+1) }{|S|}+\Delta^2 \big) \big]$ and inductively using (\ref{eqn_step3}), we can further drive the final upper bound of $\mathbb{E}[f(\x_{k+1})]$ as:
$
	\mathbb{E}[f(\x_K)]-\mathbb{E}[f(\x_0)] \leq 
	 -\sum_{i=0}^{K-1}(\beta_1\mathbb{E} [ \left \| \v_i \right\|^2 ])+K\eta\epsilon_1^2,
$                                                         
which further implies that:
\begin{align} \label{eqn_step9}
	&\mathbb{E} [f(\x^*)] - \mathbb{E}[f(\x_0)] \leq  -\sum_{i=0}^{K-1}(\beta_1\mathbb{E} [\left \| \v_i \right\|^2]) + K \eta \epsilon_1^2.
\end{align}
Rearranging terms in (\ref{eqn_step9}) yields:
\begin{align}
	\mathbb{E} [\left \|\v_\zeta \right\|^2] = \frac{1}{K}\sum_{i=0}^{K-1}\mathbb{E} [\left \|\v_i\right\|^2] \leq \frac{f(\x_0)-f^*}{K\beta_1}+\frac{\eta}{\beta_1}\epsilon_1^2.
\end{align}

\smallskip
{\em Step 3):}  By combining results in Steps~1) and 2) and plugging them into (\ref{eqn_step1}), we arrive at:
\begin{multline} \label{eqn_step6}
\mathbb{E} \big[ \left \| \nabla f(\x_\zeta) \right\|^2 \big] \leq \left[\frac{2}{\beta_1} \left(\eta+{2L^2} \left(\Delta^2+\frac{q(\Delta+1)}{|S|} \right)\eta^3 \right) + 4\right]\epsilon_1^2\\+ 
\left[ \frac{{4L^2}\left(\Delta^2+\frac{q(\Delta+1)}{|S|}\right)\eta^2}{K\beta_1}+\frac{2}{K\beta_1} \right]  ({f(\x_0)-f^*}).
\end{multline}                                                             
Lastly, we choose the following parameter: 
$
q=\sqrt{N}, \quad S=\sqrt{N}, \quad \eta\leq\frac{1}{4L(\Delta+1)}. 
$                                                          
Plugging the above parameters in the definitions of $\beta_{1}$, we obtain that:
\begin{align}
\beta_1&= \frac{\eta}{2}-\frac{L\eta^2}{2}- L^2\eta^3\left(\frac{q(\Delta+1)}{|S|}+\Delta^2 \right) >0.
\end{align}
For $mod(k,q)=0$, we have $\mathbb{E} \big[ \left \| \v_k-\nabla f(\x_k) \right\|^2 \big] =0$.
Then, after $K$ iterations, we have:
$$
\mathbb{E} \big[ \left \| \nabla f(\x_\zeta) \right\|^2 \big] \leq 16L(\Delta+1)\frac{(9\Delta^2+17\Delta+9)   (f(\x_0)-f^*) }{K\cdot(7\Delta^2 +13\Delta+5)} .
$$
Solving for $K$ to to ensure that $\mathbb{E} \big[ \left \| \nabla f(\x_\zeta) \right\|^2  \big]\leq \epsilon^2$, we have
\begin{align*}
K = \mathcal{O} \left(  \frac{ (f(\x_0)-f^*)(\Delta+1) }{\epsilon^2} \right).
\end{align*}
This completes the first part of the theorem.

Lastly, to show the SFO complexity, note that the number of SFO calls in the outer loops can be calculated as $\lceil \frac{K}{q} \rceil N$.
Also, the number of SFO calls in the inner loop can be calculated as $KS$.
Hence, the total SFO complexity can be calculated as:
$\lceil \frac{K}{q} \rceil N + K\cdot S \leq  \frac{K+q}{q}N + K\sqrt{N}= K\sqrt{N}+N+K\sqrt{N}=O(\sqrt{N}\epsilon^{-2}(\Delta+1) +N)$.
This completes the proof.
\end{proof}

\subsection{Generalization Performance Analysis} \label{subsec:sa_spiderboost_dm_gen}

After studying the convergence performance of \algname under the distributed memory architecture, in this subsection, we turn our attention to the generalization performance of \algname under distributed memory, i.e., how accurate a model trained by \algname is when it is fed by new data outside of the training dataset.
Formally, generalization error can be defined as follows.
Let $F_{N}(\x) \triangleq \min_{\x \in \mathbb{R}^{d}} \frac{1}{N}\sum_{i=1}^{N} f (\x;\xi_i)$ represent the finite-sample empirical risk minimization (ERM). 
The minimum empirical risk above is a sample-average proxy for the minimum population risk, i.e., $F(\x) \triangleq \min_{\x \in \mathbb{R}^{d}} \mathbb{E}_{\xi\in \mathcal{D}} f (\x;\xi)$, where the sample $\xi$ is drawn from an underlying latent distribution $\mathcal{D} $. 
Let $A$ represent an algorithm ($A$ could potentially be randomized) and let $S$ represent a training dataset used by $A$.
Then, the generalization error of the algorithm $A$ can be defined as the gap between the ERM problem and the population risk minimization problem: $ |\mathbb{E}_{S,A} [F_{N}[A(S)]-F[A(S)]|$.
Our generalization analysis is based on the notion of algorithmic stability~\cite{hardt2016train}, which is restated as follows:
\begin{defn} [Algorithmic Stability~\cite{hardt2016train}] \label{defn_alg_stab}
Let $S$ and $S'$ be two datasets that differ by at most one element. 
For $\epsilon' > 0$, an algorithm $A$ is said to be $\epsilon'$-uniformly stable if $\sup_{\xi\in\mathcal{D}} \mathbb{E}_A [f(A(S);\xi)-f(A(S');\xi)]\leq \epsilon' $, where $\mathcal{D}$ is the distribution of data sample $\xi$.
\end{defn} 
It is shown~\cite{bousquet2002stability} that if an algorithm $A$ is $\epsilon'$-uniformly stable, then the average generalization error of $A$ is bounded as $|\mathbb{E}_{S,A} [F_{N}[A(S)]-F[A(S)]] |\leq \epsilon' $.
Here, our goal is to study the stability of \algname under the distributed memory architecture.  
Recall that the update rule for \algname  is given by 
\begin{align*}
\x_{k+1}&=\x_{k}-\eta \v_k \\
&=\x_{k}-\eta \frac{1}{|S|} \sum_{i\in S} (\nabla f(\x_{k-\tau(k)},\xi_i)\\
&\quad -\nabla f(\x_{k-\tau(k)-1-\tau(k-\tau(k)-1)},\xi_i)+\v_{k-1-\tau(k)}).
\end{align*}
Consider two datasets $S=\{\xi_1,\xi_2,...,\xi_N\}$ and $S'=(\xi_1',\xi_2',...,\xi_N')$ that differ by at most one sample. 
The next two theorems state our main results on algorithmic stability (or equivalently, generalization error) of \algname for convex and non-convex loss problems, respectively.
Due to space limitation, the proof details are provided in the online technical report of this paper \cite{proof}.

\begin{theorem}[Algorithm Stability]   \label{generalization_distributed} 
Under assumptions \ref{bnded_f}-\ref{assum_bnded_delay}, 
the distributed-memory-based \algname is $(2 \eta M^2 K^2+ (2M^2K^2\eta)/N)$-uniformly stable.
In addition to the above conditions, if $f(x ;\xi_i)$ is quadratic and $\mu$-strongly convex, the distributed-memory-based \algname is $(2\eta M^2 K)/N$-uniformly stable.
\end{theorem}

\begin{remark}{\em
Two remarks on Theorems~\ref{generalization_distributed} are in order:
i) Although seemingly ideal, the quadratic strongly convex and smooth setting remains of practical interest.
For example, the square loss and linear learning model naturally fits the quadratic strongly convex and smooth setting.
On the other hand, many other learning problems nowadays (e.g., deep neural networks) adopt highly non-convex models, which can be covered by the result in Theorem~\ref{generalization_distributed}.
ii) The distributed-memory-based \algname under both convexity settings generalizes better as the number of training data samples $N$ increases or as the learning rate $\eta$ decreases.
Since a small learning rate implies slower convergence, there is a fundamental {\em trade-off between training convergence speed and generalization performance}.  
Furthermore, the stronger convexity condition leads to tighter stability bound (i.e., generalizes better). 
}\end{remark}
We note that our generalization analysis offers the first theoretical understanding of generalization performance for semi-asynchronous variance-reduced learning algorithms.
Our proof technique is different from existing works. 
The conventional idea used in existing works is to analyze the difference between $ \|\x_{k+1}-\x_{k+1}' \|$ and $ \|\x_{k}-\x_{k}' \|$, where $k$ and $k'$ denotes the outputs of the algorithm on datasets $S$ and $S' $, respectively. 
In contrast, our proof technique in Theorem~\ref{generalization_distributed} is inspired by the linear control system analysis, where we analyze the difference between $ \delta_{k+1}\triangleq\x_{k+1}-\x_{k+1}' $ and $ \delta_k\triangleq\x_{k}-\x_{k}' $ directly.
This helps us obtain a tighter stability error bound in Theorem~\ref{generalization_distributed}. 
Also, the variance reduction component, the asynchrony in the algorithm, and the non-convexity also create challenges in analyzing the stability performance of \algname. 

\begin{proof}[Proof Sketch of Theorem~\ref{generalization_distributed}]
Due to space limitation, we provide a proof sketch here and refer readers to our technical report~\cite{proof} for details.
Let $S=(\xi_1,\xi_2,...,\xi_N)$ and $S'=(\xi_1',\xi_2',...,\xi_N')$ be two adjacent datasets that differ by at most one element.
Define $\delta_k\triangleq \x_k-\x_k' $. 
Next, we structure our proof in several major steps.
We let $\x_0=\x_0'$ and start with evaluating $\mathbb{E} [\delta_{k+1}]$.

\smallskip
{\em Step 1)  Simplifying the expression of $\delta_{k+1} $:} Since $\v_k$ is an unbiased estimate of $\nabla f(\x_{k-\tau(k)})$, we have
\begin{align*}
\mathbb{E} [\delta_{k+1}] &=\mathbb{E} [\x_{k+1}] -\mathbb{E} [\x_{k+1}'] \\
& =\mathbb{E} [\delta_k] -\eta [\mathbb{E} [\v_k] -\mathbb{E} [\v_k'] ]\\
& =\mathbb{E} [\delta_k] -\eta [\mathbb{E} [\nabla f(\x_{k-\tau(k)})] -\mathbb{E} [\nabla f(\x_{k-\tau(k)}')] ].
\end{align*}
At Step $k$, with probability $1-\frac{1}{N}$, the sample is the same in $S$ and $S'$. 
Also, with probability $\frac{1}{N}$, the sample is different in $S$ and $S'$.  
Based on the update rule of $\v_k$, we have:
\begin{align} \label{eqn_nonconvex_step2}
&\mathbb{E} [\| (\delta_{k+1})\|] \leq \mathbb{E} \bigg[ \bigg \| (\delta_{k}) -\eta  \frac{1}{N}\sum_{i=1}^N\nabla f(\x_{k-\tau(k)},\xi_i) \notag\\
&+\eta \frac{1}{N}\sum_{i=1}^N \nabla f(\x_{k-\tau(k)}',\xi_i) \bigg\| \bigg] +\mathbb{E} \bigg[ \bigg \|\frac{1}{N}\eta\nabla f(\x_{k-\tau(k)}',\xi_i)  \notag\\
& -\frac{1}{N}\eta\nabla f(\x_{k-\tau(k)}',\xi_i')  \bigg\| \bigg]
\overset{(a)}{\leq}\mathbb{E} [\| (\delta_{k}) \| ]+2 \eta M+\frac{2\eta M}{N },
\end{align}
where $(a)$ follows from the bounded gradient assumption.

\smallskip
{\em Step 2) Bounding the error $\mathbb{E} [\|\delta_{k+1} \|]$: } 
For those $k$-values that satisfy $mod(k,q)=0$, we have $\|\delta_{k+1} \|=\| \x_{k+1} -\x_{k+1}' \| \leq \| \delta_{k} \|+\frac{2\eta M}{N}$. 
Also, from (\ref{eqn_nonconvex_step2}), we always have $\| \delta_{k+1} \|\leq \| \delta_{k} \|+2 \eta M+\frac{2\eta M}{N}$ for $mod(k,q) \ne 0$. 
By applying this bound inductively, we can bound $\delta_{k+1}$ using the total number $K$  iterations as:
\begin{align}
\mathbb{E} [\| \delta_{k+1} \| ] \leq  2 \eta M K+\frac{2\eta MK}{N }.
\end{align}
{\em Step 3)  Bounding the $\epsilon'$-stability of \algname in the non-convex case:}	
Lastly, it follows from the definition of algorithmic stability and the $M$-Lipschitz assumption of the loss function that our \algname algorithm has the following stability error bound:
\begin{align*}
\epsilon'\leq M\cdot\mathbb{E}\parallel \delta_{K+1}\parallel \leq  2 \eta M^2 K+\frac{2\eta M^2K}{N }.
\end{align*}
This completes the proof of stability bound with the non-convex loss function of Theorem~\ref{generalization_distributed}.
Next, we will provide the proof of the condition with quadratic strongly convex loss function.

\smallskip
{\em Step 4)  Simplifying the expression of $\delta_{k+1} $:}	
Since $\v_k$ is an unbiased estimate of $\nabla f(\x_{k-\tau(k)})$, we have:
\begin{align*}
\mathbb{E} [\delta_{k+1}] &=\mathbb{E}  [\x_{k+1}] -\mathbb{E} [\x_{k+1}'] \\
&=\mathbb{E} [\delta_k] -\eta [\mathbb{E} [\v_k] -\mathbb{E} [\v_k'] ] \nonumber\\
&=\mathbb{E} [\delta_k] -\eta [\mathbb{E} [\nabla f(\x_{k-\tau(k)})] -\mathbb{E} [\nabla f(\x_{k-\tau(k)}')] ].
\end{align*}
At Step $k$, with probability $1-1/N$, the training sample is the same in $S$ and $S'$. On the other hand, with probability $\frac{1}{N}$, the training sample is different between $S$ and $S'$. 
Next, we define 
\begin{align*}
\epsilon'' \triangleq \mathbb{E} \big[\frac{1}{N}\eta\nabla f((\x_{k-\tau(k)}',\xi_i))   -\frac{1}{N}\eta\nabla f((\x_{k-\tau(k)}',\xi_i'))  \big]. 
\end{align*}
Based on the update rule of $\v_k$ 
and  the quadratic property, we can show that:
\begin{align} \label{eqn_convex_step2}
	&	\mathbb{E}  (\delta_{k+1}) 
		\leq \mathbb{E} \big[(\delta_{k}) -\eta  \frac{1}{N}\sum_{i=1}^N \nabla f((\x_{k-\tau(k)},\xi_i)) +\eta \frac{1}{N}\sum_{i=1}^N \nonumber\\
		& \nabla f((\x_{k-\tau(k)}',\xi_i)) \big] +\mathbb{E} \big[\frac{1}{N}\eta\nabla f((\x_{k-\tau(k)}',\xi_i))   -\frac{1}{N}\eta \nonumber\\
		& \nabla f((\x_{k-\tau(k)}',\xi_i'))  \big]
		{\leq} \mathbb{E}\big[(\delta_{k}) -\eta  A (\delta_{k-\tau(k)})\big]+\epsilon''.
\end{align}

\smallskip
{\em Step 5) Bounding $\mathbb{E}[\|\delta_{k+1} \|]$: } A key novelty in theoretical analysis of this paper is that, based on the relation between $\mathbb{E} [\delta_{k}]$, $\mathbb{E} [\delta_{k-\tau(k)}]$, and $\mathbb{E} [\delta_{k+1}]$ proved in Step 1), we transform the problem to a {\em linear control system}, which lifts the state space into a higher dimensional space to  bound $\mathbb{E} [\|\delta_{k+1} \|]$ as follows:
\begin{align*}
	\left[
	\begin{matrix}
		\delta_{k+1}\\
		\delta_{k} \\
		...\\
		\delta_{k-\tau(k)+1} \\
		\delta_{k-\tau(k)} \\
	\end{matrix}
	\right]=\underbrace{ 
	\left[
	\begin{matrix}
		1 & 0 & \cdots &- \eta A & 0\\
		1&0& \cdots &0&0\\
		\vdots & \vdots & \ddots & \vdots & \vdots \\ 
		\vdots & \vdots & & \ddots & \vdots \\
		0&0& \cdots &1&0\\
	\end{matrix} \right]}_{\Q}
	\left[
	\begin{matrix}
		\delta_{k}  \\
		\delta_{k-1} \\
		...\\
		\delta_{k-\tau(k)}  \\
		\delta_{k-\tau(k)-1} \\
	\end{matrix}
	\right]+
	\left[
	\begin{matrix}
		\epsilon''\\
		0 \\...\\0\\0
	\end{matrix}
	\right].
\end{align*}

For convenience, we define the coefficient matrix as $\Q$.
We then show that the maximum eigenvalue of $\Q$ satisfies $\max(\lambda_{\Q}|)=1$ after some algebraic manipulations.
Since $\|\epsilon''\|\leq \frac{2 \eta M}{N}$, for $mod(k,q) \ne 0$ we have:
$
		\mathbb{E} [\| \delta_{k+1}\|] \leq  \mathbb{E}[ \| \delta_{k}\|] + \frac{2\eta M}{N}.
$
For those $k$-values that satisfy $mod(k,q)=0$, we can also show $\|\delta_{k+1} \| = \|\x_{k+1} - \x_{k+1}' \| \leq \| \delta_{k} \| + \frac{2\eta M}{N}$. 
By applying this bound inductively, we can bound $\delta_{k+1}$ using the total number $K$  iterations as:
\begin{align}
\mathbb{E} [\| \delta_{k+1}\| ] \leq \frac{2\eta M K}{N}. 
\end{align}

\smallskip
{\em Step 6)  Bounding the $\epsilon'$-stability of \algname in the quadratic strongly-convex case: 
}
Lastly, from the definition of algorithmic stability and the $M$-Lipschitz assumption of the loss function, our \algname algorithm has the following stability bound:
\begin{align}
\epsilon'\leq M\cdot\mathbb{E}\parallel \delta_{K+1}\parallel \leq \frac{2\eta M^2K}{N}.
\end{align}
This completes the proof of Theorem~\ref{generalization_distributed}.
\end{proof}

\section{\algnamens: Shared Memory} \label{sec:shared_mem}

In this section, we turn our attention to the \algname algorithm for the shared memory architecture.
We will first describe our algorithm in Section~\ref{subsec:sa_spiderboost_sm} and then present its convergence results in Section~\ref{subsec:sa_spiderboost_sm_convergence}.
Lastly, we will analyze the generalization performance of \algname for the shared memory architecture in Section~\ref{subsec:sa_StoRe_sm_gen}.

\subsection{Algorithm Description} \label{subsec:sa_spiderboost_sm}

Recall that the shared memory architecture usually models cases where a single machine with multiple cores/GPUs sharing the same memory. 
In the shared memory architecture, we have the parameter $\x$ stored in the shared memory space, and there are $P$ threads that can access it. 
Each thread reads the freshest value of $\x$, denoted as $\x_{new}$, from the shared memory. 
Then, each thread randomly chooses any mini-batch $S$ of samples and locally computes a stochastic gradient 
$$\frac{1}{|S|} \sum_{i\in S} (\nabla f(\x_{new},\xi_i)-\nabla f(\x_{old},\xi_i)+\v_{old}).$$
All threads are allowed equal access to the shared memory and can update each individual component at will. 
Each thread updates its current parameters based on the asynchronous stochastic gradient information. 
Note that, while a single thread updates the parameters in the shared memory, the parameters may have been updated by other threads with their stochastic gradient information based on more recent $\x$-values. 
Hence, this update could turn out to be “delayed.” 
To avoid race conditions in the shared memory, only a single coordinate of $\x$ can be updated at a time~\cite{recht2011hogwild}.
Let $[\a]_i$ represent the $i$-th index of vector $\a$.
Hence, in each iteration, the update of parameter $\x$ can be written as follows:
\begin{align*}
[\x_{k+1}]_{m_k}=[\x_{k}]_{m_k}- \eta [\v_{k}]_{m_k},
\end{align*} 
where $m_k\in \{1,2,...,d\}$ represents the updated coordinate in $\x$ in iteration $k$, $\eta$ is a constant learning rate, $\v_k$ represents the variance-reduced gradient, which can also be written as following similar arguments as in Eq.~\eqref{eqn_vk}:

\begin{align*}
\v_k=\frac{1}{|S|} \sum_{i\in S} (\nabla f(\x_{k-\tau(k)},\xi_i)&-\nabla f(\x_{k-\tau(k)-1-\tau(k-\tau(k)-1)},\xi_i)\notag\\&+\v_{k-1-\tau(k)}). 
\end{align*}

We illustrate the \algname algorithm for the shared memory architecture in Algorithm~\ref{alg:sa_spiderboost_sm}.
Note that, compared to the shared-memory \algname algorithm, the algorithmic structure of Algorithm~\ref{alg:sa_spiderboost_sm} is similar.
The major differences  are:
i) there is no separation of server and worker codes due to the fact that only a single machine executes the code;
and ii) only one coordinate can be updated at a time to avoid race conditions (cf. Step~13).

\begin{algorithm}[t!] 
	\SetAlgoLined
	\SetNoFillComment
	\KwIn{$ q, K \in \mathbb{N}$}
	\For{$k=0,1,...,K-1$}{
		\If{$mod(k,q)=0$}{
			Compute full gradient $\v_k=\nabla f(\x_k)=\frac{1}{N} \sum_{i=1}^{N}\nabla f(\x_{k} ,\xi_i)$\\
			Set $\x_{old}=\x_k$ and $\v_{old} = \v_k$.\\			
			Send a signal to all threads to interrupt all unfinished  jobs at each thread. \\
		}
		\Else {
			{\em /* Parallel Computation on muliple Threads: */}\\
			If receiving interrupt signal at any time, go immediately to Step~9.
			
	    		Read the parameter $\x_{old}$ and $\v_{old}$ from the shared memory. Set the most fresh parameter of $x$ as $\x_{new}$.
			 
	    		Select a subset $S$ of samples from $N$ samples uniformly at random. 
			
			Compute $ \v_{new} =\frac{1}{|S|} \sum_{i\in S} (\nabla f(\x_{new},\xi_i)-\nabla f(\x_{old},\xi_i)+\v _{old})$,  which takes combined $\tau_k$ time-slots in computation and communication.
			
			Let $\x_{old}=\x_{new}$ and  $\v_{old}=\v_{new}$.
		}
	Select $m_k$ from $\{1,...,d\}$ uniformly at random; Update $(\x_{k+1})_{m_k}=(\x_{k})_{m_k}- \eta (\v_{old})_{m_k}$.
	}
	\KwOut{$\x_{\zeta}$, where the index $\zeta$ is chosen uniformly at random from $\{1,\ldots,K\}$.}
	\caption{The \algname Algorithm for the Shared Memory Architecture.}
	\label{alg:sa_spiderboost_sm}
\end{algorithm}

\subsection{Convergence Performance Analysis} \label{subsec:sa_spiderboost_sm_convergence}

In this subsection, we first present the main convergence result of \algname under the shared memory architecture in Theorem~\ref{bbb}.
Due to the similarity to the proof of Theorem~\ref{aaa} and space limitation, we omit the proof here and relegate the proof of Theorem~\ref{bbb} to our online technical report~\cite{proof}.

\begin{theorem}[Convergence of \algname for Shared Memory] \label{bbb}
Let $f^{*}$ denote the global optimal value.
Under shared memory and assumptions \ref{bnded_f}-\ref{assum_bnded_delay}, for some $\epsilon >0$, by choosing parameters $ q=|S|=\sqrt{N}$ and $\eta\leq\frac{1}{2L(\Delta+1)}$ for a given maximum delay $\Delta>0$, the \algname algorithm outputs an 
$\x_\zeta$ that satisfies $\mathbb{E}[ \| \nabla f(\x_\zeta) \|^2 ]\leq\epsilon^2$ if 
$K = \mathcal{O} \big(  \frac{ f(\x_0)-f^* }{\epsilon^2} \big)$, which also implies an $\mathcal{O} (\sqrt{N}\epsilon^{-2}(\Delta+1)d+N)$ total SFO complexity. 
\end{theorem}

\begin{remark}{\em
The computational complexity results in Theorem~\ref{bbb} has an extra $d$-factor compared to that in Theorem~\ref{aaa}. 
This is because, under the shared memory architecture, one is allowed to read and write only a single coordinate of $\x$ at a time to avoid race conditions, i.e., 
the update rule of parameter $\x_k$ is $(\x_{k+1})_{m_k}=(\x_{k})_{m_k}- \eta (\v_{old})_{m_k}$, where $m_k\in \{1,2,...,d\}$ is a randomly selected updated coordinate of $\x$ in iteration $k$. 
Thus, we have $\mathbb{E} [ \| {[\x_{k+1}]}_{m_k}-{[\x_k}]_{m_k}\|^2 ] \leq \frac{1}{d} \mathbb{E}[ \|\eta {[\v_{old}]}_{m_k}\|^2]$ rather than  $\mathbb{E} [\| \x_{k+1}-\x_k\|^2 ] = \mathbb{E} [\|\eta \v_{old}\|^2]$. 
This is the intuition behind the existence of an extra $d$-factor in the stated results in Theorem~\ref{bbb}.
}\end{remark}

\subsection{Generalization Performance Analysis} \label{subsec:sa_StoRe_sm_gen}

In this section, our goal is to study the algorithmic stability of \algname under the shared memory architecture.  
Recall that the update rule of \algname for shared memory is given by 
\begin{align*}
& \vspace{-.2in} [\x_{k+1}]_{m_k} \!\!=\!\! [\x_{k}]_{m_k} \!\!-\! \eta [\v_k]_{m_k} \!\!=\! [\x_{k}]_{m_k} \!\!-\! \eta \frac{1}{|S|} \! \sum_{i\in S} [\nabla f(\x_{k-\tau(k)},\xi_i]_{m_k}\\
&\vspace{-.2in} -\nabla f(\x_{k-\tau(k)-1-\tau(k-\tau(k)-1)},\xi_i)_{m_k}+(v_{k-1-\tau(k)})_{m_k}),
\end{align*}
where $m_k \in \{1,2,...,d\}$ is the updated coordinate in $\x$ in iteration $k$, and we run the update iteratively for a maximum number of $K$ iterations. 
We consider two adjacent datasets $S=\{\xi_1,\xi_2,...,\xi_N\}$ and $S'=(\xi_1',\xi_2',...,\xi_N')$ from the same space, which differ in at most one element. 
For the shared memory architecture, we obtain the following main result on the algorithmic stability (or equivalently, generalization error) for \algname:

%
%
%
%
%
%
%

\begin{theorem}[Algorithm Stability] \label{generalization_shared_Theorem} 
Under assumptions \ref{bnded_f}-\ref{assum_bnded_delay}, the shared-memory version of the \algname algorithm is $((2 \eta M^2 K)/\sqrt{d}+(2M^2K\eta)/N\sqrt{d})$-uniformly stable.
In addition to the above conditions, suppose that the loss function $f(x ;\xi_i)$ is quadratic and $\mu$ -strongly convex.
Then, the shared-memory version of the \algname algorithm is $( (2\eta M^2 K)/N\sqrt{d}))$-uniformly stable.
\end{theorem}

\begin{figure*}[t!]
	\centering
	\subfigure[Performance on the CIFAR-10 under distr.-mem. system.]{
		\includegraphics[width=0.21\linewidth]{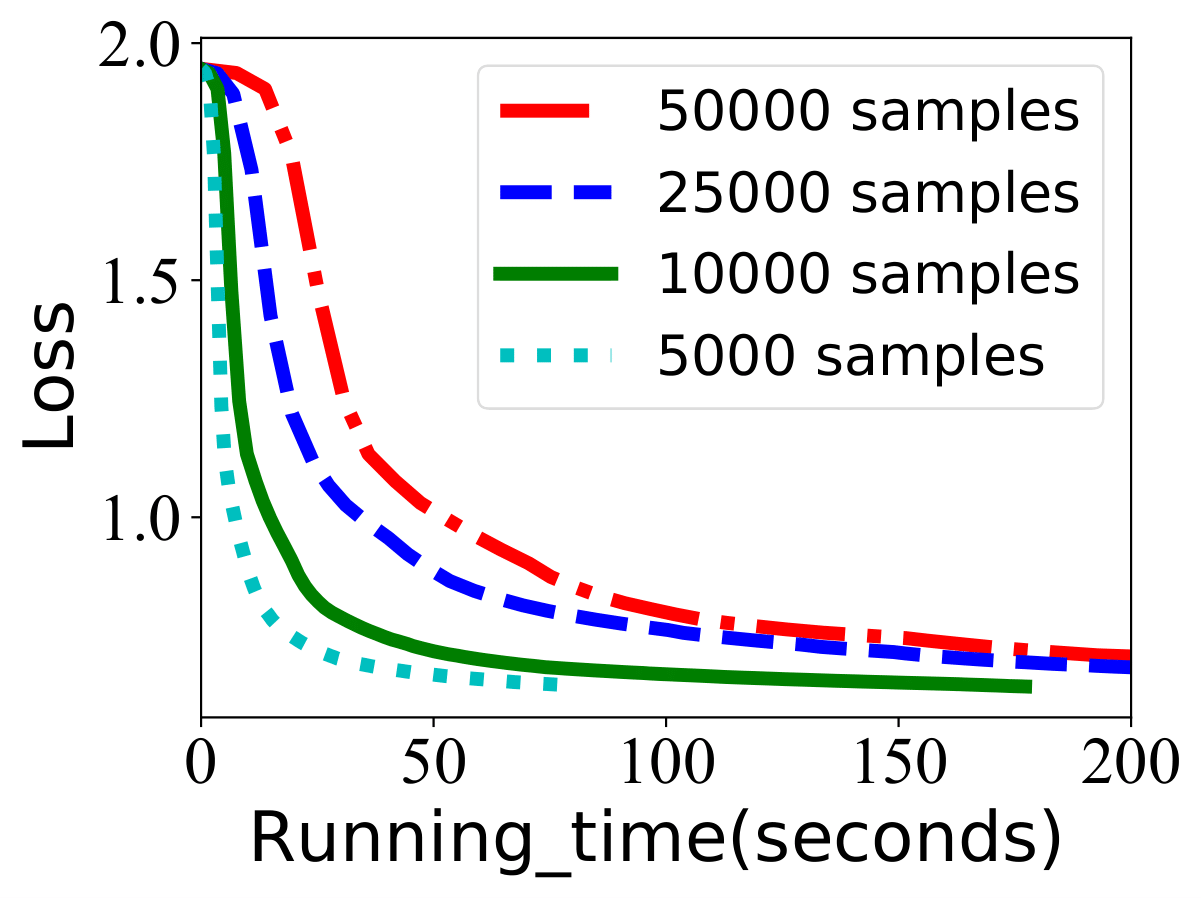} 
		\includegraphics[width=0.21\linewidth]{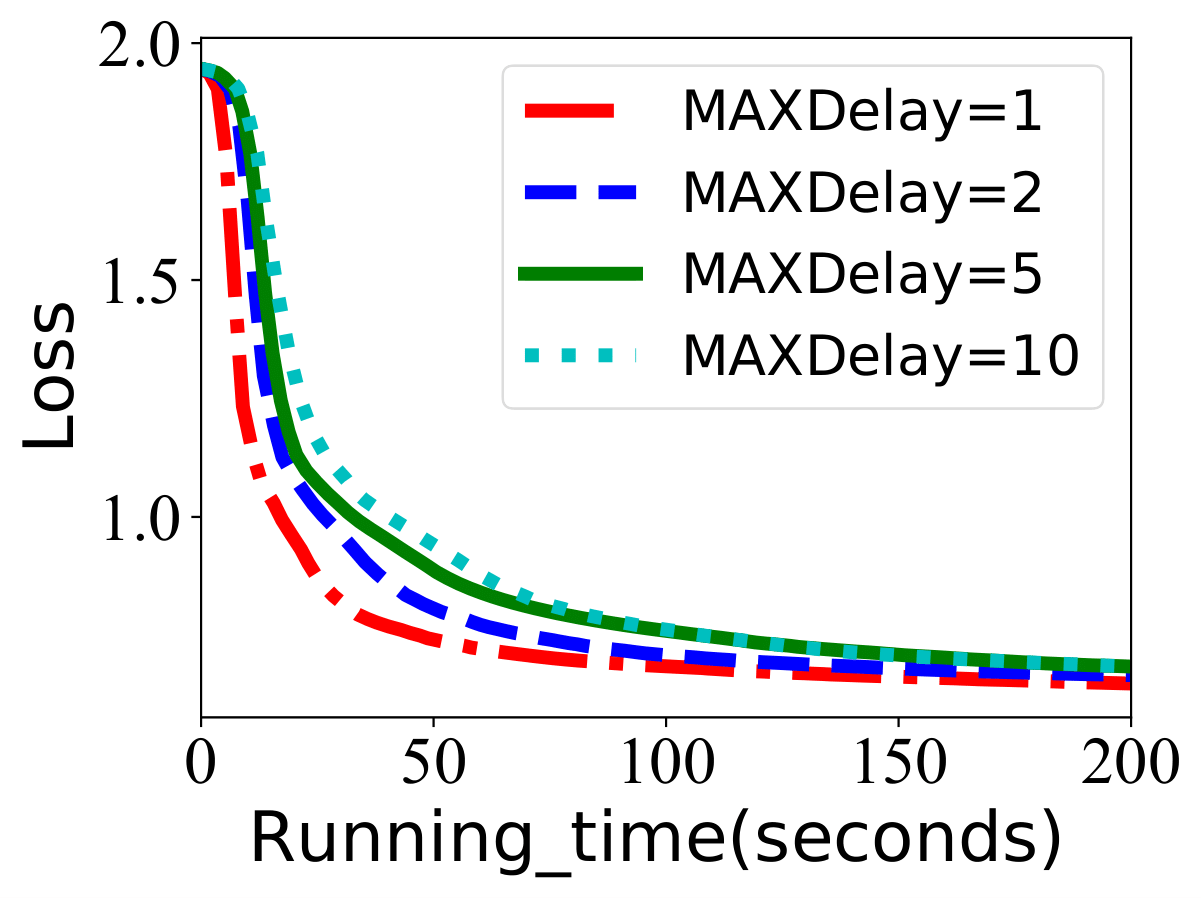}	
	}
	\subfigure[Performance on the MNIST under the shared-mem. system.]{
		\includegraphics[width=0.21\linewidth]{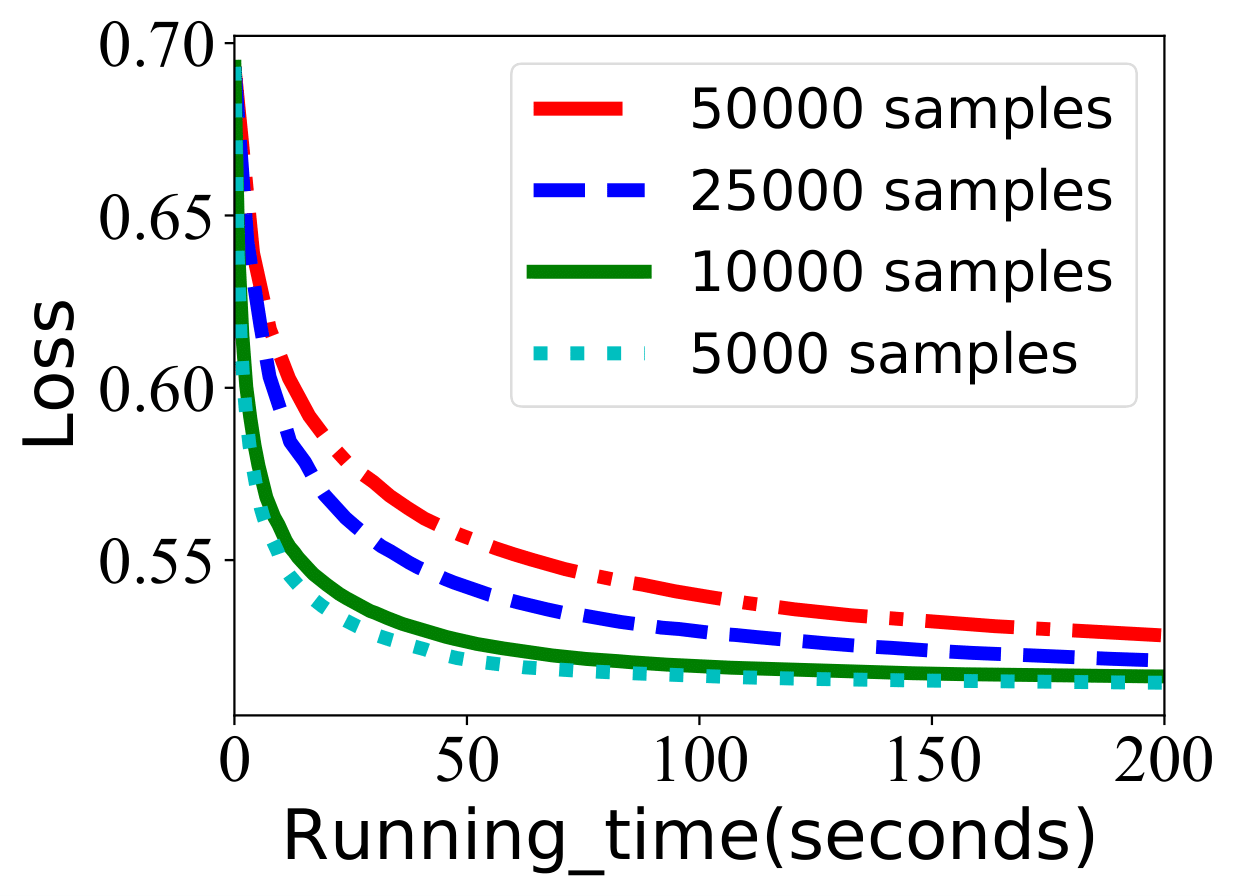} 
		\includegraphics[width=0.21\linewidth]{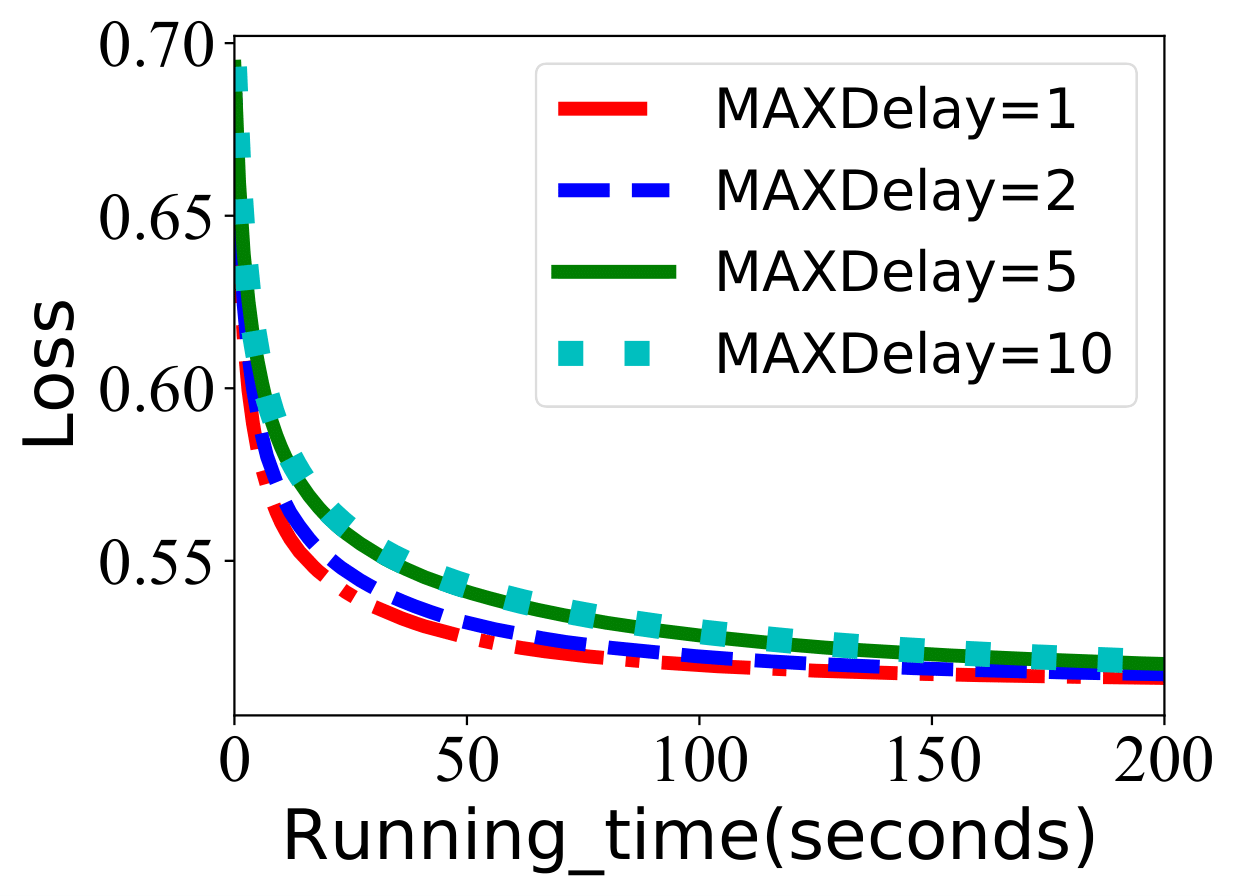}		
	}
	\caption{The convergence performance of \algname. }
	\label{fig:1} 
\end{figure*}

\begin{figure*}[h]
	\centering
	\subfigure[CIFAR-10 under distr.-mem. system.]{
		\includegraphics[width=0.2\linewidth]{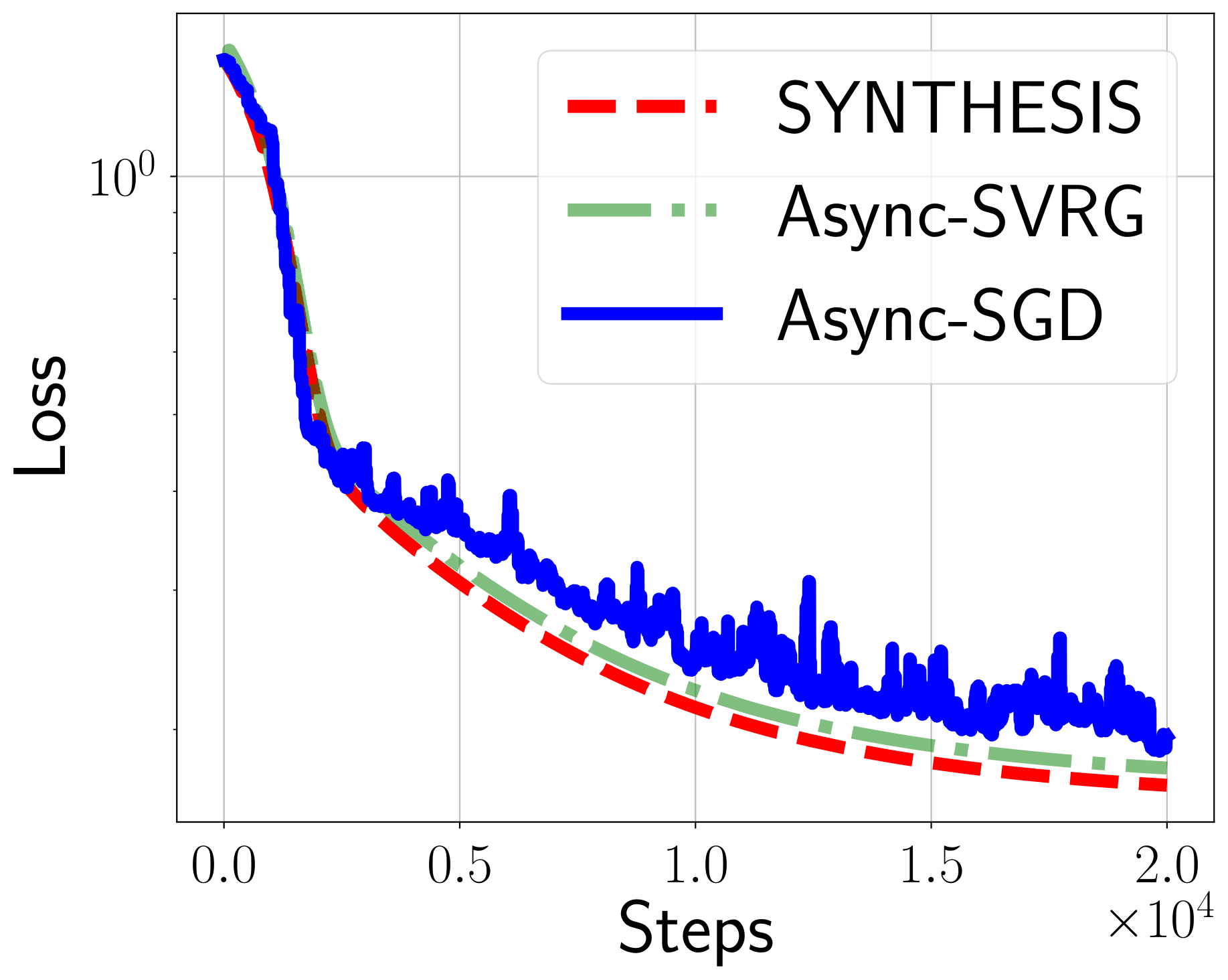} 
	}
	\subfigure[MNIST under distr.-mem. system.]{
		\includegraphics[width=0.2\linewidth]{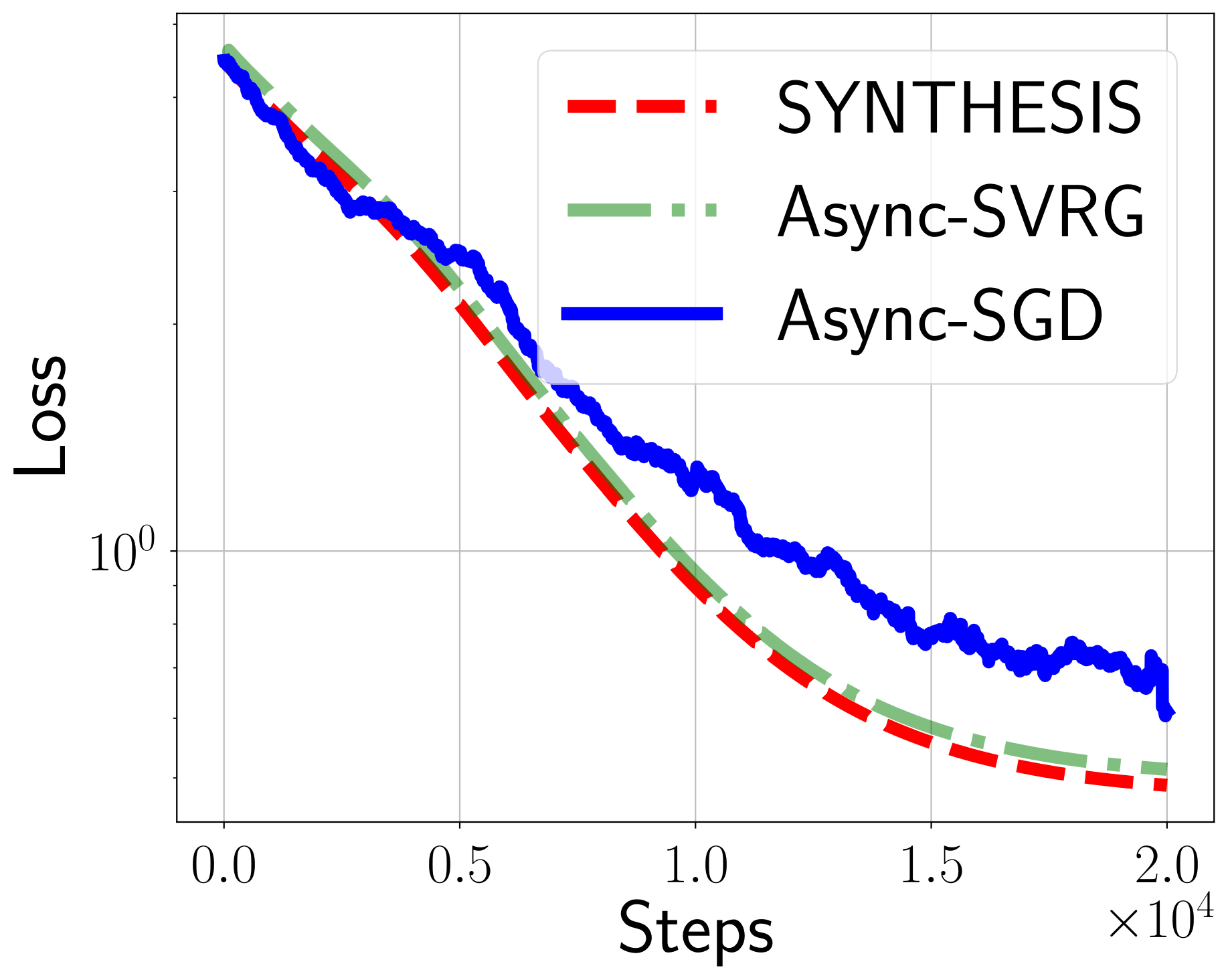}	
	}
	\subfigure[CIFAR-10 under shared-mem. system.]{
		\includegraphics[width=0.22\linewidth]{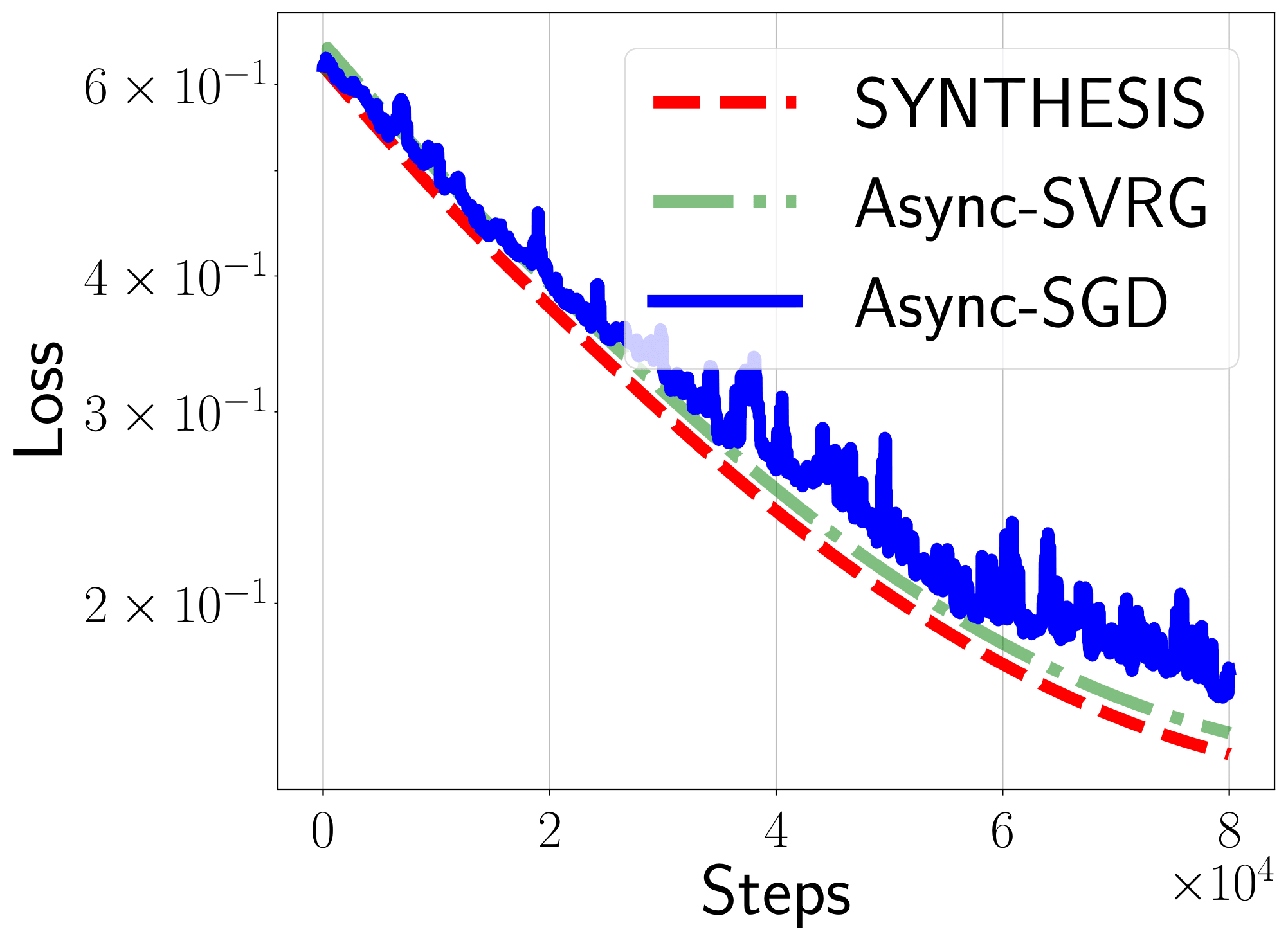} 
	}
	\subfigure[MNIST under shared-mem. system.]{
		\includegraphics[width=0.215\linewidth]{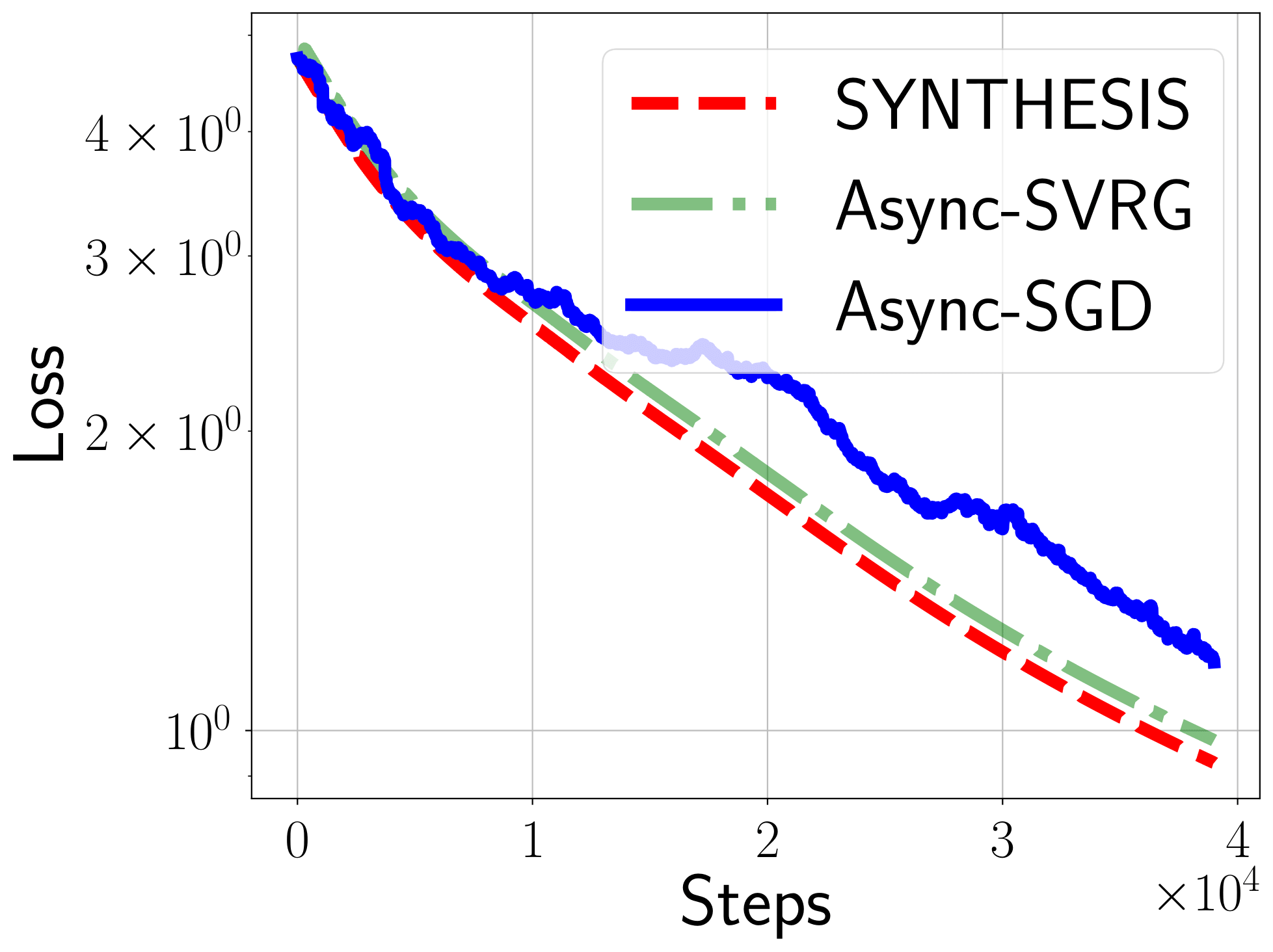}		
	}
	\caption{The convergence performance of different algorithms. }
	\label{fig:2} 
\end{figure*}

\begin{remark}{\em
Theorems~\ref{generalization_shared_Theorem} provides the algorithmic stability results for both quadratic strongly-convex and non-convex settings. 
The proof techniques are similar to those of Theorem~\ref{generalization_distributed}, and we omit the proofs here for brevity. 
We can see that the stability error bounds in Theorems~\ref{generalization_shared_Theorem} have an extra $(1/\sqrt{d})$-factor compared to those in Theorems~  \ref{generalization_distributed}. 
This is because the shared memory architecture only allows reading and writing a single coordinate of $\x$ at a time to avoid race conditions.
Therefore, we have 
$$\mathbb{E} [ \| [\nabla f(\x_{k-\tau(k)}',\xi_i)]_{m_k}  - [\nabla f(\x_{k-\tau(k)}',\xi_i')]_{m_k}  \| ] \leq  \frac{ 2M}{\sqrt{d}},$$ 
hence the existence of an extra $(1/\sqrt{d})$-factor.
This extra $(1/\sqrt{d})$-factor in the denominator implies that, as the larger the dimensionality $d$ increases, the generalization performance of our \algname algorithm becomes better in the shared memory setting.
}
\end{remark}

\section{Numerical Results}\label{sec:numerical}

In this section, we conduct numerical experiments to verify our theoretical findings of the \algname algorithms under distributed memory and shared memory architectures, respectively. 
First, we will test the convergence performance of \algname under the distributed memory architecture and the shared memory architecture. 
Then, we will illustrate the generalization performance via simple logistic regression using the Breast-Cancer-Wisconsin dataset. 
In particular, we evaluate and compare the generalization and convergence performance metrics with two state-of-the-art algorithms in asynchronous first-order methods:

\begin{list}{\labelitemi}{\leftmargin=1em \itemindent=-0.09em \itemsep=.2em}
	\item {\em Async-SGD}~\cite{zhang2018taming}: This algorithm has a single-loop architecture, where each worker randomly picks a mini-batch from its local dataset at each iteration to compute a stochastic gradient and updates the parameter server in an asynchronous fashion.
	Here, the maximum delay due to asynchrony is chosen to be the same $\Delta$ as that in the simulation of our \algname algorithm.
	\item {\em Async-SVRG}~\cite{huo2016asynchronous}: This method has  a double-loop architecture and employs unbiased gradient estimates. 
	At the beginning of each outer loop iteration, a full pass over all $N$ training samples is used to compute a full gradient. 
	This full gradient is then used in its associated inner loop to adjust the stochastic gradients.
\end{list}


%
{\bf 1) Experiment Settings:}
For the distributed memory architecture, we conduct experiments on the Amazon Web Service (AWS) platform. 
There are four workers and one parameter server. 
For the shared memory architecture, we conduct experiments on a single machine with multiple cores.
We leverage {\em OpenMPI} (Message Passing Interface) to facilitate the shared memory parallelism. 
There are four threads in the platform working in this experiment.  

We use a pre-trained convolutional neural network (CNN) model with the CIFAR-10 dataset~\cite{abadi2016tensorflow} and MNIST dataset~\cite{lecun2010mnist}.
We randomly select a subset of dataset with a varying size $N$ as the training dataset and fix the step-size at $\eta=0.1$ and choose the mini-batch size as $\lceil \sqrt{N}  \rceil$. 
In the CIFAR-10 dataset, each sample is of dimension $d = 384$. Thus, we construct a fully connected three-layer $(384\times 100 \times10)$ neural network with the ReLU activation function and the Softmax loss function while training CIFAR-10. 
In the MNIST dataset, each sample is a vector of dimension 784. 
Thus, we build a fully connected three-layer $(784\times 100 \times10)$ neural network with  the ReLU activation and the Softmax loss for MNIST.

\smallskip
{\bf 2) Numerical Results:}
We first examine the convergence of the training loss value of the \algname algorithms with different training dataset size $N$ and different values of maximum delay $\Delta$ on each worker. 
Fig.~\ref{fig:1}(a) illustrates the convergence performance of training loss value with respect to the running time of the distributed-memory version of the \algname algorithm  with different $N$ and $\Delta$ values, while Fig.~\ref {fig:1}(b) illustrates the performance of the shared-memory version of \algnamens. 
It can be seen that as $N$ or $\Delta$ decreases, the convergence rates of both distributed-memory and shared-memory versions of the \algname algorithm increase.
This is consistent with our theoretical algorithm complexity results. 
Note that the above algorithm complexity results imply an $O(1/K)$ convergence rate, which can also be observed in Fig.~\ref{fig:1}.
This behavior makes intuitive sense because a larger number of training samples implies a heavier computation load.  
Also, a larger maximum delay value of $\Delta$ results in more stale stochastic gradient information, and thus inducing a negative impact on the convergence rate performance. 
Fig.~\ref{fig:2} illustrate the algorithms comparison results under both distributed and shared-memory architectures.
We plot and compare the performances of Async-SGD, Async-SVRG, and \algname.
We choose the same starting points for all algorithms, which are generated randomly from a normal distribution.
Specifically, we choose a constant step-sizes $10^{-2}$, training sample size $N=5000$, and mini-batch $S$ with $|S|=\lceil \sqrt{N}  \rceil$. 
Fig.~\ref{fig:2} show that \algname has a faster convergence than both Async-SGD and Async-SVRG.
In the CIFAR-10 dataset case under distributed memory system, we can see that SYNTHESIS needs around $1.5\times 10^4$ steps to achieve the same loss as the Async-SVRG algorithm after running $2\times 10^4$ steps. 
Since we use the same length of training window and batch size in \algname and Async-SVRG, \algname has lower sample complexity than Async-SVRG. Similar results can be concluded in other cases in Fig.~\ref{fig:2}. 

For the algorithms generalization performance, we set datasets that differ only by one sample based on the definition of algorithmic stability. 
To do so, we first choose a subset $S$ of samples from the entire dataset and then construct a perturbed subset $S'$ by replacing the last sample in $S$ with a randomly selected sample from the whole dataset.
Finally, we run our optimization algorithms to compute the normalized Euclidean distance between $\x_k$ and $\x_k'$. 
Table.~\ref{table} illustrates the algorithmic stability results of different algorithms with $10^3$ training iterations.
Table~\ref{table} shows a slightly worse stability performance (an approximate $10^{-4}$ gap with $10^{3}$ iterations) than SGD and a similar stability performance as that of Aysnc-SVRG. 
The results show that Async-SGD has the smallest bound on the algorithmic stability, implying better generalization.
This is consistent with the widely observed generalization robustness of SGD (see, e.g.,~\cite{Hassibi19:ICLR}).
Moreover, the performance of \algname under distributed memory is worse than that of the shared memory architecture, which again confirms our theoretical findings.

\begin{table}[t!]
		\caption{Algorithms stability comparison.}
		\begin{tabular}{ |c|c|c|c| } 
			\hline
			System&Async-SGD& Async-SVRG&\algname\\
			\hline
			Distributed Mem. &$6.0\times10^{-4}$& $7.6\times10^{-4}$ & $8.3\times10^{-4}$\\ 
			Shared Mem. & $5.1\times10^{-4}$	& $6.2\times10^{-4}$& $6.7\times10^{-4}$ \\ 
			\hline
		\end{tabular}
		\label{table}
	\vspace{-.1in}
\end{table}

\section{Conclusion}\label{sec:conclusion}

In this paper, we proposed an algorithm called \algname (\ul{s}emi-as\ul{yn}chronous pa\ul{th}-int\ul{e}grated \ul{s}tochastic grad\ul{i}ent \ul{s}earch) for non-convex distributed learning under distributed memory and shared memory architectures in computing clusters.
We showed that our \algname algorithm achieves $O(\sqrt{N}\epsilon^{-2}\Delta+N)$ and $O(\sqrt{N}\epsilon^{-2}\Delta d+N)$ computational complexities under these two architectures, respectively. 
We also provided the stability error upper bounds for the \algname algorithm under distributed memory and shared memory architectures. 
We showed that smaller step-size and more training samples improve the algorithmic stability, while larger maximum delays in asynchronous updates have a negative impact on the convergence performance. 
Collectively, our results in this work advance the understanding of the convergence and generalization performances of semi-asynchronous distributed first-order variance-reduced methods.


This work has been supported in part by NSF grants CAREER CNS-2110259, CNS-2112471, CNS-2102233, and CCF-2110252.

\bibliographystyle{acm}
\bibliography{./Zhuqing_Async}

\onecolumn
\appendix

\section{Convergence analysis of SA-SpiderBoost for distributed memory } \label{appdx_proof_aaa}
To prove the stated in Theorem~\ref{aaa}, we first prove a useful lemma.
\begin{lemma} \label{lemma1}
	Let all assumptions hold and apply SA-SpiderBoost in Algorithm \ref{alg:sa_sb_dm_server} and Algorithm \ref{alg:sa_sb_dm_worker}, if the parameters $\eta ,q $ and $S$ are chosen such that 
\begin{equation}\label{eq:4}
\begin{aligned}
\beta_1\triangleq  &\frac{\eta}{2}-\frac{L\eta^2}{2}- L^2\eta^3\big(\frac{q(\Delta+1) }{|S|}+\Delta^2 \big) > 0,
\end{aligned}                                                             
\end{equation}
and if for $mod(k,q)=0$,we always have
\begin{equation}\label{eq:5}
\begin{aligned}
\mathbb{E} \big \| \v_k-\nabla f(\x_k) \big\|^2\leq\epsilon_1^2,
\end{aligned}                                                             
\end{equation}
then the output point  $\x_\zeta$ of SA-SpiderBoost satisfies  
\begin{equation}\label{eq:6}
\begin{aligned}
\mathbb{E}\big \| \nabla f(\x_\zeta) \big\|^2 \!\leq\! \big[\frac{2}{\beta_1}(\eta \!+\! {2L^2}(\Delta^2 \!+\! \frac{q(\Delta+1)}{|S|})\eta^3 ) \!+\! 4\big]\epsilon_1^2 \!+\! \big[ \frac{{4L^2}(\Delta^2+\frac{q(\Delta+1)}{|S|})\eta^2}{K\beta_1} \!+\! \frac{2}{K\beta_1} \big]  ({f(\x_0) \!-\! f^*}).\!\!
\end{aligned}                                                             
\end{equation} 
\end{lemma}
\subsection{Proof of Lemma \ref{lemma1}} 
\begin{proof}

In Lemma \ref{lemma1}, we aim to bound $\mathbb{E}\big \| \nabla f(\x_\zeta) \big\|^2$, where $\zeta$ is chosen from $\{1,...,K\}$ uniformly at random. 
Following the inequality of arithmetic and geometric means, we have:
\begin{equation}\label{eq:13}
\begin{aligned}
&\mathbb{E}\big \| \nabla f(\x_\zeta) \big\|^2=\mathbb{E}\big \| \nabla f(\x_\zeta)-\v_\zeta+\v_\zeta \big\|^2\leq  2\mathbb{E}\big \| \v_\zeta \big\|^2+2\mathbb{E}\big \| \nabla f(\x_\zeta)-\v_\zeta\big\|^2.
\end{aligned}                                                             
\end{equation}   
Next, we bound the terms $\mathbb{E}\left \| \v_\zeta \right\|^2$ and $\mathbb{E} \left \| \nabla f(\x_\zeta) - \v_\zeta\right\|^2$ on the right-hand-side of (\ref{eq:13}) individually.

To evaluate $\mathbb{E} \left \| \nabla f(\x_\zeta) - \v_\zeta\right\|^2$, we start from the iteration relationship of our SA-SpiderBoost algorithm, for which we have:
\begin{align}\label{eq:3}
\mathbb{E}& \big \| \v_{k}-\frac{1}{N}\sum_{i=1}^{N}\nabla f(\x_{k},\xi_i) \big\|^2\nonumber\\&\overset{(a)}{\leq} 2\mathbb{E} \big \| \v_{k}-\frac{1}{N}\sum_{i=1}^{N}\nabla f(\x_{k-\tau(k)},\xi_i) \big\|^2 +2\mathbb{E} \big \|\frac{1}{N}\sum_{i=1}^{N}\nabla f(\x_{k-\tau(k)},\xi_i)-\frac{1}{N}\sum_{i=1}^{N}\nabla f(\x_{k},\xi_i) \big\|^2 
\nonumber\\
&\overset{(b)}{\leq} 2\mathbb{E} \big \| \v_{k}-\frac{1}{N}\sum_{i=1}^{N}\nabla f(\x_{k-\tau(k)},\xi_i) \big\|^2 +2L^2\mathbb{E} \big \| \x_{k-\tau(k)}-\x_{k} \big\|^2 
\nonumber\\
&= 2\mathbb{E} \big \| \v_{k}-\frac{1}{N}\sum_{i=1}^{N}\nabla f(\x_{k-\tau(k)},\xi_i) \big\|^2 +2L^2\mathbb{E} \big \| \sum_{j=k-\tau(k)}^{k-1} (\x_{j+1}- \x_{j})\big\|^2 
\nonumber\\
&\overset{(c)}{\leq} 2\mathbb{E} \big \| \v_{k}-\frac{1}{N}\sum_{i=1}^{N}\nabla f(\x_{k-\tau(k)},\xi_i) \big\|^2 +2L^2\Delta \sum_{j=k-\tau(k)}^{k-1} \mathbb{E} \big \|  (\x_{j+1}- \x_{j})\big\|^2 
\nonumber\\
&\overset{(d)}{=} 2\mathbb{E} \big \| \v_{k}-\frac{1}{N}\sum_{i=1}^{N}\nabla f(\x_{k-\tau(k)},\xi_i) \big\|^2 +2L^2\eta^2\Delta \sum_{j=k-\tau(k)}^{k-1} \mathbb{E} \big \|  \v_{j}\big\|^2,
\end{align}                                                                           
where $(a)$ follows from the triangle inequality, $(b)$ follows from $L$-smooth property,
$(c)$ is due to the triangle inequality and the maximum delay is $\Delta$, 
and $(d)$ uses the condition of update rule.

Next, we bound the terms $ \mathbb{E} \big \| \v_{k}-\frac{1}{N}\sum_{i=1}^{N}\nabla f(\x_{k-\tau(k)},\xi_i) \big\|^2 $ on the right-hand-side of (\ref{eq:3}), which denotes the distance of the inner loop. Let $n_k=\lceil k/q \rceil$ such that $(n_k-1)q\leq k \leq n_kq -1$, we have

\begin{align}\label{eq:1}
\mathbb{E} &\big \| \v_{k}-\frac{1}{N}\sum_{i=1}^{N}\nabla f(\x_{k-\tau(k)},\xi_i) \big\|^2\nonumber\\
\overset{(a)}{=}  &\mathbb{E} \big\|\frac{1}{|S|}  \sum_{i\in S} \big[\nabla f(\x_{k-\tau(k)},\xi_i)-\nabla f(\x_{k-1-\tau(k)-\tau^{k-1-\tau(k)}},\xi_i)  -\frac{1}{N}\sum_{j=1}^{N}\nabla f(\x_{k-\tau(k)},\xi_j)\nonumber\\&+  \frac{1}{N}\sum_{j=1}^{N}\nabla f(\x_{k-1-\tau(k)-\tau^{k-1-\tau(k)}},\xi_j)  \big]\big \|^2+\mathbb{E}\big \| \frac{1}{|S|} \sum_{i\in S} \big [ \v_{k-1-\tau(k)} -\frac{1}{N}\sum_{j=1}^{N}\nabla f(\x_{k-1-\tau(k)-\tau^{k-1-\tau(k)}},\xi_j)  \big] \big\|^2\nonumber\\
\overset{(b)}{\leq} &\frac{1}{|S|^2} \sum_{i\in S}  \mathbb{E} \big \|\big [\nabla f(\x_{k-\tau(k)},\xi_i)-\nabla f(\x_{k-1-\tau(k)-\tau^{k-1-\tau(k)}},\xi_i)  \big]-\frac{1}{N}\sum_{j=1}^{N}\nabla f(\x_{k-\tau(k)},\xi_j)\nonumber\\&+ \frac{1}{N}\sum_{j=1}^{N}\nabla f(\x_{k-1-\tau(k)-\tau^{k-1-\tau(k)}},\xi_j) \big\|^2+\mathbb{E} \big \|  \v_{k-1-\tau(k)} -\frac{1}{N}\sum_{i=1}^{N}\nabla f(\x_{k-1-\tau(k)-\tau^{k-1-\tau(k)}},\xi_i)  \big\|^2\nonumber\\
\overset{(c)}{\leq} &\frac{1}{|S|^2} \sum_{i\in S} \mathbb{E}\big \|\nabla f(\x_{k-\tau(k)},\xi_i)-\nabla f(\x_{k-1-\tau(k)-\tau^{k-1-\tau(k)}},\xi_i) \big\|^2\nonumber\\&+\mathbb{E}\big \|  \v_{k-1-\tau(k)} -\frac{1}{N}\sum_{i=1}^{N}\nabla f(\x_{k-1-\tau(k)-\tau^{k-1-\tau(k)}},\xi_i)   \big\|^2\nonumber\\
\overset{(d)}{\leq} &\frac{L^2}{|S|^2} \sum_{i\in S}  \mathbb{E} \big \|\x_{k-\tau(k)}-\x_{k-1-\tau(k)-\tau^{k-1-\tau(k)}}\big\|^2+\mathbb{E}\big \|  \v_{k-1-\tau(k)} -\frac{1}{N}\sum_{i=1}^{N}\nabla f(\x_{k-1-\tau(k)-\tau^{k-1-\tau(k)}},\xi_i)   \big\|^2\nonumber\\
\overset{(e)}{=} &\frac{L^2\eta^2}{|S|}\mathbb{E} \big \|  \sum_{i=k-\tau(k)-1-\tau(k-\tau(k)-1)}^{k-\tau(k)-1}  \v_{i}\big\|^2+\mathbb{E}\big \|  \v_{k-1-\tau(k)} -\frac{1}{N}\sum_{i=1}^{N}\nabla f(\x_{k-1-\tau(k)-\tau^{k-1-\tau(k)}},\xi_i)   \big\|^2\nonumber\\
\overset{(f)}{\leq}&\frac{L^2\eta^2}{|S|}(\Delta+1) \sum_{i=k-\tau(k)-1-\tau(k-\tau(k)-1)}^{k-\tau(k)-1} \mathbb{E} \big \|   \v_{i}\big\|^2+\mathbb{E}\big \|  \v_{k-1-\tau(k)} -\frac{1}{N}\sum_{i=1}^{N}\nabla f(\x_{k-1-\tau(k)-\tau^{k-1-\tau(k)}},\xi_i)   \big\|^2,
\end{align}                                                                            
where $(a)$ follows from the gradient update rule $$\v_k=\frac{1}{|S|} \sum_{i\in S} (\nabla f(\x_{k-\tau(k)},\xi_i)-\nabla f(\x_{k-1-\tau(k)-\tau^{k-1-\tau(k)}},\xi_i) +\v_{k-1-\tau(k)}),
$$  and  Lemma 1 in \cite{fang2018spider}, $(b)$ and $(c)$ use in \cite[Appendix~A.3]{fang2018spider},  
$(d)$ follows from the Lipschitz continuity of $\nabla f(\x)$, 
$(e)$ is due to the condition on update rule,
and $(f)$ follows from the maximum delay is $\Delta$.

Since $\v_{k-1-\tau(k)}$ are generated from the previous step. Telescoping $q$ iterations, we obtain that:
\begin{align}\label{eq:2}
\mathbb{E} \big \| \v_{k}-\frac{1}{N}\sum_{i=1}^{N}\nabla f(\x_{k-\tau(k)},\xi_i) \big\|^2\leq &\frac{L^2\eta^2(\Delta+1) }{|S|}\sum_{j=(n_k-1)q}^{k-1-\tau(k)} \mathbb{E}\big \| \v_{j}\big\|^2\nonumber\\&+ \mathbb{E} \big \|\v_{(n_k-1)q}-\frac{1}{N}\sum_{i=1}^{N}\nabla f(\x_{(n_k-1)q },\xi_i)\big\|^2.
\end{align}

By combining (\ref{eq:3})  and (\ref{eq:2}), we arrive at:

\begin{align} \label{eq:step_32}
\mathbb{E}& \big \| \v_{k}-\frac{1}{N}\sum_{i=1}^{N}\nabla f(\x_{k},\xi_i) \big\|^2\nonumber\\&{\leq}\frac{2L^2\eta^2(\Delta+1) }{|S|}\sum_{j=(n_k-1)q}^{k-1-\tau(k)} \mathbb{E}\big \| \v_{j}\big\|^2+2 \mathbb{E} \big \|\v_{(n_k-1)q}-\frac{1}{N}\sum_{i=1}^{N}\nabla f(\x_{(n_k-1)q },\xi_i)\big\|^2\nonumber\\&+2L^2\eta^2\Delta \sum_{j=k-\tau(k)}^{k-1} \mathbb{E} \big \|  \v_{j}\big\|^2.
\end{align}

Next, we continue  to bound the other term $\mathbb{E}\left \| \v_\zeta \right\|^2$ on the right-hand-side of (\ref{eq:13}). 
By Assumption \ref{assum_Lip_grad}, the entire objective function $f$ is $L-$smooth, which further implies 

\begin{align}\label{eq:7}
f(\x_{k+1})&\leq f(\x_k)+\langle\nabla f(\x_k),\x_{k+1}-\x_{k}\rangle+\frac{L}{2}\big \| \x_{k+1}-\x_{k} \big\|^2\nonumber\\
&\overset{(a)}{=}  f(\x_k)-\eta \langle\nabla f(\x_k),\v_k\rangle+\frac{L\eta^2}{2}\big \| \v_k \big\|^2\nonumber\\
&= f(\x_k)-\eta \langle\nabla f(\x_k)-\v_k,\v_k\rangle-\eta  \big \|\v_k \big\|^2+\frac{L\eta^2}{2}\big \| \v_k \big\|^2\nonumber\\
&\overset{(b)}{\leq}  f(\x_k)+\frac{\eta}{2} \big \| \v_k-\nabla f(\x_k) \big\|^2-(\frac{\eta}{2}-\frac{L\eta^2}{2})\big \| \v_k \big\|^2,
\end{align}                                                                            
where $(a)$ follows from the update rule of SA-Spiderboost, $(b)$ uses   uses the inequality that $\langle\x,\y\rangle\leq\frac{\| \x \|^2 + \| \y \|^2}{2}$, for $\x,\y\in \mathbb{R} ^d$. 
Thus, we have
\begin{align}\label{eq:8}
\mathbb{E}&f(\x_{k+1})\leq\mathbb{E} f(\x_k)+\frac{\eta}{2} \mathbb{E}\big \| \v_k-\nabla f(\x_k) \big\|^2-(\frac{\eta}{2}-\frac{L\eta^2}{2})\mathbb{E}\big \| \v_k \big\|^2\nonumber\\
\overset{(a)}{\leq}&\frac{\eta}{2}\big (  \frac{2L^2\eta^2(\Delta+1) }{|S|}\sum_{j=(n_k-1)q}^{k-1-\tau(k)} \mathbb{E}\big \| \v_{j}\big\|^2+2 \mathbb{E} \big \|\v_{(n_k-1)q}-\frac{1}{N}\sum_{i=1}^{N}\nabla f(\x_{(n_k-1)q },\xi_i)\big\|^2\nonumber\\
&+2L^2\eta^2\Delta \sum_{j=k-\tau(k)}^{k-1} \mathbb{E} \big \|  \v_{j}\big\|^2 \big)-	(\frac{\eta}{2}-\frac{L\eta^2}{2})\mathbb{E}\big \| \v_k \big\|^2+\mathbb{E} f(\x_k)\nonumber\\
\overset{(b)}{\leq}&\mathbb{E} f(\x_k)-(\frac{\eta}{2}-\frac{L\eta^2}{2})\mathbb{E}\big \| \v_k \big\|^2+\eta\epsilon_1^2+\frac{L^2\eta^3(\Delta+1) }{|S|}\sum_{j=(n_k-1)q}^{k-1-\tau(k)} \mathbb{E}\big \| \v_{j}\big\|^2+L^2\eta^3\Delta \sum_{j=k-\tau(k)}^{k-1} \mathbb{E} \big \|  \v_{j}\big\|^2\nonumber \\
\leq &\mathbb{E} f(\x_k)-(\frac{\eta}{2}-\frac{L\eta^2}{2})\mathbb{E}\big \| \v_k \big\|^2+\eta\epsilon_1^2+\frac{L^2\eta^3(\Delta+1) }{|S|}\sum_{j=(n_k-1)q}^{k} \mathbb{E}\big \| \v_{j}\big\|^2+L^2\eta^3\Delta \sum_{j=k-\tau(k)}^{k-1} \mathbb{E} \big \|  \v_{j}\big\|^2 ,
\end{align}             
where $(a)$ follows from (\ref{eq:3}), $(b) $ follows from the $\mathbb{E} \big \| \v_{(n_k-1)q}-\nabla f(\x_{(n_k-1)q}) \big\|^2 \leq \epsilon_1^2$.
Next, telescoping (\ref{eq:8}) over $k$ from $(n_k-1)q$ to $k$, where $k\leq n_kq-1 $, we have:
\begin{align}\label{eq:9}
\mathbb{E}f(\x_{k+1})\overset{(a)}{\leq} &\mathbb{E} f(\x_{(n_k-1)q})-(\frac{\eta}{2}-\frac{L\eta^2}{2})\sum_{i=(n_k-1)q}^k\mathbb{E}\big \| \v_i \big\|^2+\sum_{i=(n_k-1)q}^k\eta\epsilon_1^2\nonumber\\
&+\frac{L^2\eta^3(\Delta+1) }{|S|}\sum_{j=(n_k-1)q}^k \sum_{i=(n_k-1)q}^{j} \mathbb{E}\big \| \v_{i}\big\|^2+L^2\eta^3\Delta \sum_{j=(n_k-1)q}^k\sum_{i=j-\tau^j}^{j-1} \mathbb{E} \big \|  \v_{i}\big\|^2\nonumber \\
\overset{(b)}{\leq} &\mathbb{E} f(\x_{(n_k-1)q})-(\frac{\eta}{2}-\frac{L\eta^2}{2})\sum_{i=(n_k-1)q}^k\mathbb{E}\big \| \v_i \big\|^2+\sum_{i=(n_k-1)q}^k\eta\epsilon_1^2\nonumber\\&+\frac{L^2\eta^3(\Delta+1) }{|S|}\sum_{j=(n_k-1)q}^k \sum_{i=(n_k-1)q}^{k} \mathbb{E}\big \| \v_{i}\big\|^2+L^2\eta^3\Delta^2 \sum_{i=(n_k-1)q}^k  \mathbb{E} \big \|  \v_{i}\big\|^2, 
\end{align}                                                             
where $(a)$ follows from (\ref{eq:8}), 
$(b)$ extends the summation of second term from $j$ to $k$.
Further relaxing (\ref{eq:9}) yields:
\begin{align}\label{eq:10}
\mathbb{E}f(\x_{k+1}) \overset{(a)}{\leq} &\mathbb{E} f(\x_{(n_k-1)q})-(\frac{\eta}{2}-\frac{L\eta^2}{2})\sum_{i=(n_k-1)q}^k\mathbb{E}\big \| \v_i \big\|^2+\sum_{i=(n_k-1)q}^k\eta\epsilon_1^2\nonumber\\&+\frac{qL^2\eta^3(\Delta+1) }{|S|}\sum_{i=(n_k-1)q}^k \mathbb{E}\big \| \v_{i}\big\|^2+L^2\eta^3\Delta^2 \sum_{j=(n_k-1)q}^k \mathbb{E} \big \|  \v_{i}\big\|^2\nonumber \\
=&\mathbb{E} f(\x_{(n_k-1)q})+\sum_{i=(n_k-1)q}^k \eta\epsilon_1^2-\big[ \frac{\eta}{2}-\frac{L\eta^2}{2}- L^2\eta^3\big(\frac{q(\Delta+1) }{|S|}+\Delta^2 \big) \big] \sum_{i=(n_k-1)q}^k\mathbb{E} \big \| \v_{i} \big\|^2\nonumber\\
=&\mathbb{E}[f(\x_{(n_k-1)q})]-\sum_{i=(n_k-1)q}^{k}(\beta_1\mathbb{E}\big \| \v_i \big\|^2-\eta\epsilon_1^2),	
\end{align}            
where $(a) $ follows from the fact that $k\leq n_kq-1$. 
Then, by telescoping, we can further derive: 
\begin{align}\label{eq:11}
\mathbb{E}f(\x_K)-\mathbb{E}f(\x_0)=&(\mathbb{E}f(\x_q)-\mathbb{E}f(\x_0))+(\mathbb{E}f(\x_{2q})-\mathbb{E}f(\x_q)+...+(\mathbb{E}f(\x_{K})-\mathbb{E}f(\x_{(n_k-1)q})))\nonumber\\
\overset{(a)}{\leq}&-\sum_{i=0}^{q-1}(\beta_1\mathbb{E}\big \| \v_i \big\|^2-\eta\epsilon_1^2)-\sum_{i=q}^{2q-1}(\beta_1\mathbb{E}\big \| \v_i \big\|^2-\eta{2}\epsilon_1^2)-...\nonumber\\
&-\sum_{(n_K-1)q}^{K-1}(\beta_1\mathbb{E}\big \| \v_i \big\|^2-\frac{\eta}{2}\epsilon_1^2)\nonumber\\
=&-\sum_{i=0}^{K-1}(\beta_1\mathbb{E}\big \| \v_i \big\|^2-\eta\epsilon_1^2)\nonumber\\=
& -\sum_{i=0}^{K-1}(\beta_1\mathbb{E}\big \| \v_i \big\|^2)+K\eta\epsilon_1^2,
\end{align}          
where $(a)$ follows from  (\ref{eq:10}). 
Since $\mathbb{E}f(\x_K)\geq f(x^*)$, then we have:
\begin{equation}\label{eq:12}
\begin{aligned}
&\mathbb{E}f(\x^*)-\mathbb{E}f(\x_0)\leq  -\sum_{i=0}^{K-1}(\beta_1\mathbb{E}\big \| \v_i \big\|^2)+K\eta\epsilon_1^2.
\end{aligned}                                                             
\end{equation}   
By rearranging (\ref{eq:12}), we have:
\begin{equation}\label{eq:14}
\begin{aligned}
\mathbb{E}\big \| \v_\zeta \big\|^2=\frac{1}{K}\sum_{i=0}^{K-1}\mathbb{E}\big \| \v_i\big\|^2\leq \frac{f(\x_0)-f^*}{K\beta_1}+\frac{\eta}{\beta_1}\epsilon_1^2,
\end{aligned}                                                             
\end{equation}
It then follows that:   
\begin{align} \label{eq:15}
\mathbb{E} \big \| \v_\zeta-\nabla f(\x_\zeta) \big\|^2&
\overset{(a)}{\leq}\mathbb{E}\frac{2L^2\eta^2(\Delta+1)}{|S|}\sum_{j=(n_{\zeta}-1)q}^{\zeta-1-\tau^{\zeta}} \mathbb{E}\big \| \v_{j}\big\|^2+2 \mathbb{E} \big \|\v_{(n_\zeta-1)q}-\frac{1}{N}\sum_{i=1}^{N}\nabla f(\x_{(n_\zeta-1)q },\xi_i)\big\|^2 \notag\\
&+2L^2\eta^2\Delta \mathbb{E}\sum_{j=\zeta-\tau^{\zeta}}^{\zeta-1} \mathbb{E} \big \|  \v_{j}\big\|^2 \nonumber\\
&\overset{(b)}{\leq}\frac{2L^2\eta^2(\Delta+1)}{|S|}\mathbb{E}\sum_{j=(n_{\zeta}-1)q}^{\zeta-1-\tau^{\zeta}} \mathbb{E}\big \| \v_{j}\big\|^2+2 \epsilon_1^2+2L^2\eta^2\Delta\mathbb{E} \sum_{i=\zeta-\tau^{\zeta}}^{\zeta-1} \mathbb{E} \big \|  \v_{i}\big\|^2 \nonumber\\
&\leq\frac{2L^2\eta^2(\Delta+1)}{|S|}\mathbb{E}\sum_{j=(n_{\zeta}-1)q}^{\xi} \mathbb{E}\big \| \v_{j}\big\|^2+2 \epsilon_1^2+2L^2\eta^2\Delta \mathbb{E}\sum_{i=\zeta-\tau^{\zeta}}^{\zeta-1} \mathbb{E} \big \|  \v_{i}\big\|^2 \nonumber\\
&\overset{(c)}{\leq}\frac{2L^2\eta^2(\Delta+1)}{|S|}\mathbb{E}\sum_{i=(n_{\zeta}-1)q}^{min\{(n_{\zeta})q-1,K-1\}} \mathbb{E}\big \| \v_{i}\big\|^2+2 \epsilon_1^2+2L^2\eta^2\Delta \mathbb{E}\sum_{i=\zeta-\tau^{\zeta}}^{\zeta-1} \mathbb{E} \big \|  \v_{i}\big\|^2 \nonumber\\
&\overset{(d)}{\leq}\frac{q}{K}\sum_{i=0}^{K-1}\frac{2L^2\eta^2(\Delta+1)}{|S|} \mathbb{E}\big \| \v_{i}\big\|^2+2 \epsilon_1^2+\frac{\Delta}{K}\sum_{i=0}^{K-1}2L^2\eta^2\Delta  \mathbb{E} \big \|  \v_{i}\big\|^2 \nonumber\\
&=(\frac{2L^2\eta^2q(\Delta+1)}{|S|}+2L^2\eta^2\Delta^2)\frac{1}{K} \sum_{i=0}^{K-1}\mathbb{E}\big \| \v_{i}\big\|^2+2 \epsilon_1^2 \nonumber\\
&\overset{(e)}{\leq}\big[\frac{{2L^2}(\Delta^2+\frac{q(\Delta+1)}{|S|})\eta^3}{\beta_1}+2\big]\epsilon_1^2+ {2L^2}(\Delta^2+\frac{q(\Delta+1)}{|S|})\eta^2\big[\frac{f(\x_0)-f^*}{K\beta_1}\big],
\end{align}
where $(a)$ follows from (\ref{eq:step_32}), 
$(b)$ follows from (\ref{eq:5}), 
$(c)$ follows from the definition of $n_{\zeta}$, which implies $\zeta \leq min\{(n_{\zeta})q-1,K-1\}$, $(d)$follows from the fact that the probability that $n_{\zeta}=1, 2,...,n_K$ is less than or equal to $\frac{q}{K}$ and $(e)$ follows from (\ref{eq:14}).

By combining (\ref{eq:14})  and (\ref{eq:15}), we arrive at:
\begin{align}\label{eq:17}
\mathbb{E}& \big \| \nabla f(\x_\zeta) \big\|^2\leq2\mathbb{E}\big \| \nabla f(\x_\zeta)-\v_\zeta\big\|^2+2\mathbb{E}\big \| \v_\zeta \big\|^2\nonumber\\
\leq & 2\big\{\big[\frac{{2L^2}(\Delta^2+\frac{q(\Delta+1)}{|S|})\eta^3}{\beta_1}+2\big]\epsilon_1^2+ {2L^2}(\Delta^2+\frac{q(\Delta+1)}{|S|})\eta^2\big[\frac{f(\x_0)-f^*}{K\beta_1}\big]\big\}\nonumber\\&+ 2[\frac{f(\x_0)-f^*}{K\beta_1}+\frac{\eta}{\beta_1}\epsilon_1^2]\nonumber\\
=&\big[\frac{2}{\beta_1}(\eta+{2L^2}(\Delta^2+\frac{q(\Delta+1)}{|S|})\eta^3 )+4\big]\epsilon_1^2+\big[ \frac{{4L^2}(\Delta^2+\frac{q(\Delta+1)}{|S|})\eta^2}{K\beta_1}+\frac{2}{K\beta_1} \big]  ({f(\x_0)-f^*}).
\end{align}                          
This completes the proof of Lemma~\ref{lemma1}.
\end{proof}

With Lemma~\ref{lemma1}, we are in a position to prove the result in Theorem~ \ref{aaa}.
Setting the parameters
\begin{equation}
\begin{aligned}
q=\sqrt{N}, \quad S=\sqrt{N}, \quad \eta=\frac{1}{4L(\Delta+1)}, 
\end{aligned}                                                             
\end{equation}  
we obtain: 
\begin{equation}
\begin{aligned}
\beta_1&\triangleq  \frac{\eta}{2}-\frac{L\eta^2}{2}- L^2\eta^3\big(\frac{q(\Delta+1)}{|S|}+\Delta^2 \big) 
=\frac{1}{8L(\Delta+1)}\cdot \frac{14\Delta^2+26\Delta+10}{16(\Delta+1)^2}>0.
\end{aligned}                                                             
\end{equation}
For $mod(k,q)=0$ we have $\mathbb{E} \big \| \v_k-\nabla f(\x_k) \big\|^2=0$. 
Then, after $K$ iterations, we have
\begin{equation}\label{eq:18}
\begin{aligned}
\mathbb{E}\big \| \nabla f(\x_\zeta) \big\|^2\leq 16L(\Delta+1)\frac{9\Delta^2+17\Delta+9}{K\cdot(7\Delta^2 +13\Delta+5)}  (f(\x_0)-f^*).
\end{aligned}                                                             
\end{equation}  
To ensure $\mathbb{E}\big \| \nabla f(\x_\zeta) \big\|\leq\epsilon$, it suffices to ensure$\mathbb{E}\big \| \nabla f(\x_\zeta) \big\|^2\leq\epsilon^2$.
Solving for $K$ yields:
$$K =16L(\Delta+1)\frac{9\Delta^2+17\Delta+9}{\epsilon^2\cdot(7\Delta^2 +13\Delta+5)} (f(\x_0)-f^*) = \mathcal{O} \big(  \frac{ f(\x_0)-f^* }{\epsilon^2} \big).$$

Lastly, to show the SFO complexity, we note that the number of SFO calls in the outer loops can be calculated as: $\lceil \frac{K}{q} \rceil N$.
Also, the number of SFO calls in the inner loop can be calculated as $KS$.
Hence, the total SFO complexity can be calculated as:
$\lceil \frac{K}{q} \rceil N + K\cdot S \leq  \frac{K+q}{q}N + K\sqrt{N}= K\sqrt{N}+N+K\sqrt{N}=O(\sqrt{N}\epsilon^{-2}(\Delta+1) +N)$.
This completes the proof.



\section{ Generalization analysis of SA-SpiderBoost for distributed memory }  \label{appdx:sa_sb_gen}
\subsection{Proofs of Theorem~\ref{generalization_distributed} }  \label{appdx:sa_sb_gen_quadr}

Recall that update rule for SA-SpiderBoost for distributed-memory system is given by:
\begin{align}
	\x_{k+1}=\x_{k}-\eta\frac{1}{|S|} \sum_{i\in S} (\nabla f(\x_{k-\tau(k)},\xi_i)-\nabla f(\x_{k-\tau(k)-1-\tau(k-\tau(k)-1)},\xi_i)+\v_{k-1-\tau(k)}) .
\end{align}
$S$ and $S'$ are two data sets such that $S$ and $S'$ differ in at most one example, where $S=(\xi_1,\xi_2,...,\xi_N)$ and $S'=(\xi_1',\xi_2',...,\xi_N')$. Let $\delta_k\triangleq \x_k-\x_k' $. Suppose $\x_0=\x_0'$.

Now, taking expectation of $\delta_{k+1}$ with respect of the algorithm, we get 
\begin{align}
		\mathbb{E} (\delta&_{k+1}) =\mathbb{E}  (\x_{k+1}) -\mathbb{E} (\x_{k+1}') 
		=\mathbb{E} (\delta_k) -\eta [\mathbb{E} (\v_k) -\mathbb{E} (\v_k') ].
\end{align}
Since we $\v_k$ is the unbiased estimated of $\nabla f(\x_{k-\tau(k)})$, we have
\begin{align}
		\mathbb{E} (\delta&_{k+1}) 
		=\mathbb{E} (\delta_k) -\eta [\mathbb{E} (\nabla f(\x_{k-\tau(k)})) -\mathbb{E} (\nabla f(\x_{k-\tau(k)}')) ].
\end{align}

At Step $k$, with probability $1-1/N$, the example is the same in  $S$ and $S'$. With probability $\frac{1}{N}$, the example is different in  $S$ and $S'$. Define $\epsilon'' \triangleq \mathbb{E} \big[\frac{1}{N}\eta\nabla f((\x_{k-\tau(k)}',\xi_i))   -\frac{1}{N}\eta\nabla f((\x_{k-\tau(k)}',\xi_i'))  \big]  $. 
We have:
\begin{align}
		\mathbb{E}  (\delta_{k+1}) 
		\leq& \mathbb{E} \big[(\delta_{k}) -\eta  \frac{1}{N}\sum_{i=1}^N\nabla f((\x_{k-\tau(k)},\xi_i)) +\eta \frac{1}{N}\sum_{i=1}^N\nabla f((\x_{k-\tau(k)}',\xi_i)) \big] \nonumber\\&+\mathbb{E} \big[\frac{1}{N}\eta\nabla f((\x_{k-\tau(k)}',\xi_i))   -\frac{1}{N}\eta\nabla f((\x_{k-\tau(k)}',\xi_i'))  \big] \nonumber \\
		\overset{(a)}{\leq}& \mathbb{E}\big[(\delta_{k}) -\eta  A (\delta_{k-\tau(k)})\big]+\epsilon'',
\end{align}
where  $(a)$ follows from $f$ is a quadratic convex function. We note that $\|\epsilon''\|\leq \frac{2\eta M}{N}$ since we have the bounded gradient Assumption \ref{assum_bnded_grad}.
Then, we have

\begin{align}
	\left[
	\begin{matrix}
		\delta_{k+1}\\
		\delta_{k} \\
		\vdots \\
		\delta_{k-\tau(k)+1} \\
		\delta_{k-\tau(k)} \\
	\end{matrix}
	\right]=
	\left[
	\begin{matrix}
		1 & 0 & \cdots &- \eta A & 0\\
		1&0& \cdots &0&0\\
		\vdots & \vdots & \ddots & \vdots & \vdots \\ 
		\vdots & \vdots & & \ddots & \vdots \\
		0&0& \cdots &1&0\\
	\end{matrix}
	\right]
	\left[
	\begin{matrix}
		\delta_{k}  \\
		\delta_{k-1} \\
		\vdots\\
		\delta_{k-\tau(k)}  \\
		\delta_{k-\tau(k)-1} \\
	\end{matrix}
	\right]+
	\left[
	\begin{matrix}
		\epsilon''\\
		0 \\ \vdots \\0\\0
	\end{matrix}
	\right].
\end{align}
Let Matrix 

\begin{align}
	\Q \triangleq
	\left[
	\begin{matrix}
		1 & 0 & \cdots &- \eta A & 0\\
		1&0& \cdots &0&0\\
		\vdots & \vdots & \ddots & \vdots & \vdots \\ 
		\vdots & \vdots & & \ddots & \vdots \\
		0&0& \cdots &1&0\\
	\end{matrix}
	\right].
\end{align}
Consider the characteristic polynomial 
\begin{align}
	\left[
	\begin{matrix}
		1 & 0 & \cdots &- \eta A & 0\\
		1&0& \cdots &0&0\\
		\vdots & \vdots & \ddots & \vdots & \vdots \\ 
		\vdots & \vdots & & \ddots & \vdots \\
		0&0& \cdots &1&0\\
	\end{matrix}
	\right]
	\left[
	\begin{matrix}
		\v_1\\
		\v_2\\
		\vdots \\ \vdots \\
		\v_{\tau(k)+2}
	\end{matrix}
	\right]
	=\lambda_\Q
	\left[
	\begin{matrix}
		\v_1\\
		\v_2\\
		\vdots \\ \vdots \\
		\v_{\tau(k)+2}
	\end{matrix}
	\right],
\end{align}
which implies $\v_1=\lambda_\Q \v_2$, $\v_2=\lambda_\Q \v_3$, ...,$\v_{\tau^{k}+1}=\lambda_\Q \v_{\tau(k)+2}$. 
Plugging this in to the first row, we have:
\begin{align}
	\big(-\lambda_\Q^{\tau^{k}+2} +\lambda_\Q^{\tau^{k}+1} + \lambda_\Q[-\eta A]  \big) \v_{\tau(k)+2}=0.
\end{align}

Since $\A$ is symmetric, then it has eigenvalue decomposition $\A = \mathbf{U} \boldsymbol{\Lambda} \mathbf{U}^{\top}$.
Then equation can be written as:
\begin{align}
	\mathbf{U} \big(-\lambda_\Q^{\tau^{k}+2} +\lambda_\Q^{\tau^{k}+1} + \lambda_\Q[-\eta \A]  \big) \mathbf{U}^{\top}=0.
\end{align}
Let $\lambda_1,...,\lambda_b \in[\mu,M]$ be eigenvalue of symmetric matrix $\A$. 
Then we have
\begin{align}
	\lambda_\Q^{\tau_k+1}\cdot \big[-\lambda_\Q+1\big]=\lambda_\Q [\eta\lambda_i].
\end{align}
Since $0< \eta<\frac{1}{ M},d \geq 1, \lambda_1,...,\lambda_b\in[\mu,M]$, then we have $0\leq \eta \lambda_i\leq1$. Thus, $\max(|\lambda_\Q|)=1$. Thus, for $mod(k,q)\neq0$ we have

\begin{align} \label{ss}
		\mathbb{E}\| \delta_{k+1}\|\leq & \mathbb{E}\| \delta_{k}\|+\frac{2\eta M}{N}. 
\end{align}
For those $k$-values that satisfy $mod(k,q)=0$, we can also show $\|(\delta_{k+1}) \|=\|(\x_{k+1}) -(\x_{k+1}') \|\leq\|(\delta_{k}) \|+\frac{2\eta M}{N}$. 
Also, from (\ref{ss}), we always have $\|(\delta_{k+1}) \|\leq \|(\delta_{k}) \|+\frac{2\eta M}{N}$ for $mod(k,q) \ne 0$. 
By applying this bound inductively, we can bound $\delta_{k+1}$ using the total number $K$  iterations as:
\begin{align}
	\parallel \delta_{k+1}\parallel\leq \frac{2\eta MK}{N}.
\end{align}
Lastly, it follows from the definition of algorithm stability and the $M$-Lipschitz Assumption \ref{assum_Lip_loss} of the loss function that our SA-SpiderBoost algorithm has the following stability bound:

$$ 
\epsilon'\leq M\cdot\mathbb{E}\parallel \delta_{t+1}\parallel \leq \frac{2\eta M^2K}{N}.
$$
This completes the proof.
\subsection{Proofs of Theorem~\ref{generalization_distributed} } \label{appdx:nonconvex_gen_dist}

Recall that update rule for SpiderBoost is given by:
\begin{align}
	\x_{k+1}=\x_{k}-\eta\frac{1}{|S|} \sum_{i\in S} (\nabla f(\x_{k-\tau(k)},\xi_i)-\nabla f(\x_{k-\tau(k)-1-\tau(k-\tau(k)-1)},\xi_i)+\v_{k-1-\tau(k)})  .
\end{align}
Let $S$ and $S'$ be two data sets such that $S$ and $S'$ differ in at most one example, where $S=(\xi_1,\xi_2,...,\xi_N)$ and$S'=(\xi_1',\xi_2',...,\xi_N')$.  
Let $\delta_k\triangleq \x_k-\x_k' $. Suppose $\x_0=\x_0'$.

Now, taking expectation of $\delta_{k+1}$ with respect of the algorithm, we get 
\begin{align}
	\begin{split}
		\mathbb{E} (\delta&_{k+1}) =\mathbb{E}  (\x_{k+1}) -\mathbb{E} (\x_{k+1}') 
		=\mathbb{E} (\delta_k) -\eta [\mathbb{E} (\v_k) -\mathbb{E} (\v_k') ].
	\end{split}
\end{align}
Since we $\v_k$ is the unbiased estimated of $\nabla f(\x_{k-\tau(k)})$, we have:
\begin{align}
	\begin{split}
		\mathbb{E} (\delta&_{k+1}) 
		=\mathbb{E} (\delta_k) -\eta [\mathbb{E} (\nabla f(\x_{k-\tau(k)})) -\mathbb{E} (\nabla f(\x_{k-\tau(k)}')) ].
	\end{split}
\end{align}
At Step $k$, with probability $1-1/N$, the example is the same in  $S$ and $S'$. With probability $\frac{1}{N}$, the example is different in  $S$ and $S'$. Hence, we have
\begin{align} \label{eqn:33}
		\mathbb{E}\| (\delta_{k+1}) \|
		\leq& \mathbb{E}\| (\delta_{k}) -\eta  (\frac{1}{N}\sum_{i=1}^N\nabla f(\x_{k-\tau(k)},\xi_i)) +\eta (\frac{1}{N}\sum_{i=1}^N\nabla f(\x_{k-\tau(k)}',\xi_i)) \|\nonumber\\&+\frac{1}{N}\eta\| \nabla f(\x_{k-\tau(k)}',\xi_i))   -\nabla f(\x_{k-\tau(k)}',\xi_i'))   \|\nonumber\\
		\overset{(a)}{\leq}& \mathbb{E}\| (\delta_{k}) -\eta  (\frac{1}{N}\sum_{i=1}^N\nabla f(\x_{k-\tau(k)},\xi_i)) +\eta (\frac{1}{N}\sum_{i=1}^N\nabla f(\x_{k-\tau(k)}',\xi_i)) \|+\frac{2\eta M}{N }\nonumber\\
		\overset{(b)}{\leq}& \mathbb{E}\| (\delta_{k}) \|+\eta  \mathbb{E}\| (\frac{1}{N}\sum_{i=1}^N\nabla f(\x_{k-\tau(k)},\xi_i)) - (\frac{1}{N}\sum_{i=1}^N\nabla f(\x_{k-\tau(k)}',\xi_i)) \|+\frac{2\eta M}{N }\nonumber\\
		\overset{(c)}{\leq}& \mathbb{E}\| (\delta_{k}) \|+2 \eta M+\frac{2\eta M}{N },
\end{align}
where $(a)$ and $(b)$ follows from the triangle inequality, $(c)$ follows from the bounded gradient Assumption~\ref{assum_bnded_grad}. 
We note that $\eta\| \nabla f(\x_{k-\tau(k)}',\xi_i))   -\nabla f(\x_{k-\tau(k)}',\xi_i'))   \|\leq 2\eta M$ derives form the bounded gradient Assumption~\ref{assum_bnded_grad}.

For those $k$-values that satisfy $mod(k,q)=0$, we have $\|(\delta_{k+1}) \|=\|(\x_{k+1}) -(\x_{k+1}') \|\leq\|(\delta_{k}) \|+\frac{2\eta M}{N}$. 
Also, from (\ref{eqn:33}), we always have $\|(\delta_{k+1}) \|\leq \|(\delta_{k}) \|+2 \eta M+\frac{2\eta M}{N}$ for $mod(k,q) \ne 0$. 
By applying this bound inductively, we can bound $\delta_{k+1}$ using the total number $K$  iterations as:
\begin{align}
	\mathbb{E}	\parallel \delta_{k+1}\parallel\leq  2 \eta M K+\frac{2\eta MK}{N }.
\end{align}

Lastly, it follows from the definition of algorithm stability and the $M$-Lipschitz  Assumption \ref{assum_Lip_loss} of the loss function that our SA-SpiderBoost algorithm has the following stability bound:
$$ 
\epsilon'\leq M\cdot	\mathbb{E} \parallel \delta_{K+1}\parallel \leq  2 \eta M^2 K+\frac{2\eta M^2K}{N }.
$$

\section{Convergence analysis of SA-SpiderBoost for shared memory } \label{appdx_proof_shared}
To prove the result stated in the theorem, we first prove a useful lemma below.
\begin{lemma} \label{lemma2}  
Let all assumptions hold and apply SA-SpiderBoost in Algorithm \ref{alg:sa_spiderboost_sm}, if the parameters $\eta ,q $ and $S$ are chosen such that 
		\begin{equation}\label{eq:23}
		\begin{aligned}
		\beta_1 \triangleq &\frac{\eta}{2d}-\frac{L\eta^2}{2d}- \frac{L^2\eta^3}{d^2}\big(\frac{q(\Delta+1)}{|S|}+\Delta^2 \big) > 0,
		\end{aligned}                                                             
		\end{equation}
and if for $mod(k,q)=0$,we always have
		\begin{equation}\label{eq:24}
		\begin{aligned}
		\mathbb{E} \big \| \v_k-\nabla f(\x_k) \big\|^2\leq\epsilon_1^2,
		\end{aligned}                                                             
		\end{equation}
		then the output point  $\x_\zeta$ of SA-SpiderBoost satisfies  
		\begin{equation}\label{eq:25}
		\begin{aligned}
		\mathbb{E}\big \| \nabla f(\x_\zeta) \big\|^2\leq\big[\frac{2}{\beta_1d}\big(\eta+\frac{2L^2}{d}(\Delta^2+\frac{q(\Delta+1)}{|S|})\eta^3 \big)\big]\epsilon_1^2+\big[ \frac{{4L^2}(\Delta^2+\frac{q(\Delta+1)}{|S|})\eta^2}{K\beta_1d}+\frac{2}{K\beta_1} \big]  ({f(\x_0)-f^*}).
		\end{aligned}                                                             
		\end{equation} 

\end{lemma} 

\begin{proof}
We aim to bound $\mathbb{E}\big \| \nabla f(\x_\zeta) \big\|^2$, $\zeta$ is random choose from array ${1,...,K}$. Following the inequality of arithmetic and geometric means, we have:
\begin{equation}\label{eq:32}
\begin{aligned}
&\mathbb{E}\big \| \nabla f(\x_\zeta) \big\|^2=\mathbb{E}\big \| \nabla f(\x_\zeta)-\v_\zeta+\v_\zeta \big\|^2\leq  2\mathbb{E}\big \| \nabla f(\x_\zeta)-\v_\zeta\big\|^2+2\mathbb{E}\big \| \v_\zeta \big\|^2.
\end{aligned}                                                             
\end{equation}   
Next, we bound the terms $\mathbb{E}\left \| \v_\zeta \right\|^2$ and $\mathbb{E} \left \| \nabla f(\x_\zeta) - \v_\zeta\right\|^2$ on the right-hand-side of (\ref{eq:32}) individually.

To evaluate $\mathbb{E} \left \| \nabla f(\x_\zeta) - \v_\zeta\right\|^2$, we start from the iteration relationship of our SA-SpiderBoost algorithm. Toward this end, we first bound the distance  of the inner loop $\mathbb{E} \big \| \v_{k}-\frac{1}{N}\sum_{i=1}^{N}\nabla f(\x_{k-\tau(k)},\xi_i) \big\|^2 $.
Let $n_k=\lceil k/q \rceil$ such that $(n_k-1)q\leq k \leq n_kq -1$, we have:
\begin{align}\label{eq:20}
\mathbb{E} &\big \| \v_{k}-\frac{1}{N}\sum_{i=1}^{N}\nabla f(\x_{k-\tau(k)},\xi_i) \big\|^2\nonumber\\
\overset{(a)}{=}  &\mathbb{E} \big \|\frac{1}{|S|}  \sum_{i\in S} \big [\nabla f(\x_{k-\tau(k)},\xi_i)-\nabla f(\x_{k-1-\tau(k)-\tau^{k-1-\tau(k)}},\xi_i)  -\frac{1}{N}\sum_{j=1}^{N}\nabla f(\x_{k-\tau(k)},\xi_j)\nonumber\\&+  \frac{1}{N}\sum_{j=1}^{N}\nabla f(\x_{k-1-\tau(k)-\tau^{k-1-\tau(k)}},\xi_j)  \big ]\big\|^2+\mathbb{E}\big \| \frac{1}{|S|} \sum_{i\in S} \big [ \v_{k-1-\tau(k)} -\frac{1}{N}\sum_{j=1}^{N}\nabla f(\x_{k-1-\tau(k)-\tau^{k-1-\tau(k)}},\xi_j)  \big] \big\|^2\nonumber\\
\overset{(b)}{\leq} &\frac{1}{|S|^2} \sum_{i\in S}  \mathbb{E} \big \|\big [\nabla f(\x_{k-\tau(k)},\xi_i)-\nabla f(\x_{k-1-\tau(k)-\tau^{k-1-\tau(k)}},\xi_i)  \big]-\frac{1}{N}\sum_{j=1}^{N}\nabla f(\x_{k-\tau(k)},\xi_j)\nonumber\\&+ \frac{1}{N}\sum_{j=1}^{N}\nabla f(\x_{k-1-\tau(k)-\tau^{k-1-\tau(k)}},\xi_j) \big\|^2+\mathbb{E} \big \|  \v_{k-1-\tau(k)} -\frac{1}{N}\sum_{i=1}^{N}\nabla f(\x_{k-1-\tau(k)-\tau^{k-1-\tau(k)}},\xi_i)  \big\|^2\nonumber\\
\overset{(c)}{\leq} &\frac{1}{|S|^2} \sum_{i\in S} \mathbb{E}\big \|\nabla f(\x_{k-\tau(k)},\xi_i)-\nabla f(\x_{k-1-\tau(k)-\tau^{k-1-\tau(k)}},\xi_i) \big\|^2\nonumber\\&+\mathbb{E}\big \|  \v_{k-1-\tau(k)} -\frac{1}{N}\sum_{i=1}^{N}\nabla f(\x_{k-1-\tau(k)-\tau^{k-1-\tau(k)}},\xi_i)   \big\|^2\nonumber \\ 
\overset{(d)}{\leq} &\frac{L^2}{|S|^2} \sum_{i\in S}  \mathbb{E} \big \|\x_{k-\tau(k)}-\x_{k-1-\tau(k)-\tau^{k-1-\tau(k)}}\big\|^2+\mathbb{E}\big \|  \v_{k-1-\tau(k)} -\frac{1}{N}\sum_{i=1}^{N}\nabla f(\x_{k-1-\tau(k)-\tau^{k-1-\tau(k)}},\xi_i)   \big\|^2\nonumber\\
\overset{(e)}{=} &\frac{L^2\eta^2}{|S|d}\mathbb{E} \big \|  \sum_{i=k-\tau(k)-1-\tau(k-\tau(k)-1)}^{k-\tau(k)-1}  \v_{i}\big\|^2+\mathbb{E}\big \|  \v_{k-1-\tau(k)} -\frac{1}{N}\sum_{i=1}^{N}\nabla f(\x_{k-1-\tau(k)-\tau^{k-1-\tau(k)}},\xi_i)   \big\|^2\nonumber \\
\overset{(f)}{\leq} &\frac{L^2\eta^2}{|S|d}(\Delta+1) \sum_{i=k-\tau(k)-1-\tau(k-\tau(k)-1)}^{k-\tau(k)-1} \mathbb{E} \big \|   \v_{i}\big\|^2+\mathbb{E}\big \|  \v_{k-1-\tau(k)} -\frac{1}{N}\sum_{i=1}^{N}\nabla f(\x_{k-1-\tau(k)-\tau^{k-1-\tau(k)}},\xi_i)   \big\|^2,
\end{align}          
where $(a)$follows from the gradient update rule $\v_k=\frac{1}{|S|} \sum_{i\in S} (\nabla f(\x_{k-\tau(k)},\xi_i)-\nabla f(\x_{k-1-\tau(k),i})+\v_{k-1-\tau(k)-\tau^{k-1-\tau(k)}})
$  and  Lemma 1 in \cite{fang2018spider}, $(b)$ and $(c)$use A.3 in \cite{fang2018spider},  $(d)$  follows from Lipchitz continuity of $\nabla f(x)$. $(e)$ is due to the condition on update rule.  $(f)$ follows from the maximum delay is $\Delta$.

Since $\v_{k-1-\tau(k)}$ are generated from previous step, telescoping $q$ iterations, we obtain that:
\begin{align}\label{eq:21}
\mathbb{E} \big \| \v_{k}-\frac{1}{N}\sum_{i=1}^{N}\nabla f(\x_{k-\tau(k)},\xi_i) \big\|^2&\leq \frac{L^2\eta^2(\Delta+1) }{|S|}\sum_{j=(n_k-1)q}^{k-1-\tau(k)} \mathbb{E}\big \| \v_{j}\big\|^2\nonumber\\&+ \mathbb{E} \big \|\v_{(n_k-1)q}-\frac{1}{N}\sum_{i=1}^{N}\nabla f(\x_{(n_k-1)q },\xi_i)\big\|^2.
\end{align}          

Next, we can conclude that,
\begin{align}\label{eq:22}
\mathbb{E}& \big \| \v_{k}-\frac{1}{N}\sum_{i=1}^{N}\nabla f(\x_{k},\xi_i) \big\|^2\nonumber\\&\overset{(a)}{\leq} 2\mathbb{E} \big \| \v_{k}-\frac{1}{N}\sum_{i=1}^{N}\nabla f(\x_{k-\tau(k)},\xi_i) \big\|^2 +2\mathbb{E} \big \|\frac{1}{N}\sum_{i=1}^{N}\nabla f(\x_{k-\tau(k)},\xi_i)-\frac{1}{N}\sum_{i=1}^{N}\nabla f(\x_{k},\xi_i) \big\|^2 
\nonumber\\&\overset{(b)}{\leq} 2\mathbb{E} \big \| \v_{k}-\frac{1}{N}\sum_{i=1}^{N}\nabla f(\x_{k-\tau(k)},\xi_i) \big\|^2 +2L^2\mathbb{E} \big \| \x_{k-\tau(k)}-\x_{k} \big\|^2 
\nonumber\\&= 2\mathbb{E} \big \| \v_{k}-\frac{1}{N}\sum_{i=1}^{N}\nabla f(\x_{k-\tau(k)},\xi_i) \big\|^2 +2L^2\mathbb{E} \big \| \sum_{j=k-\tau(k)}^{k-1} (\x_{j+1}- \x_{j})\big\|^2 
\nonumber\\&\overset{(c)}{\leq} 2\mathbb{E} \big \| \v_{k}-\frac{1}{N}\sum_{i=1}^{N}\nabla f(\x_{k-\tau(k)},\xi_i) \big\|^2 +2L^2\Delta \sum_{j=k-\tau(k)}^{k-1} \mathbb{E} \big \|  (\x_{j+1}- \x_{j})\big\|^2 
\nonumber\\&\overset{(d)}{=} 2\mathbb{E} \big \| \v_{k}-\frac{1}{N}\sum_{i=1}^{N}\nabla f(\x_{k-\tau(k)},\xi_i) \big\|^2 +\frac{2L^2\eta^2\Delta}{d} \sum_{j=k-\tau(k)}^{k-1} \mathbb{E} \big \|  \v_{j}\big\|^2 \nonumber\\
&\overset{(e)}{\leq}\frac{2L^2\eta^2(\Delta+1) }{|S|d}\sum_{j=(n_k-1)q}^{k-1-\tau(k)} \mathbb{E}\big \| \v_{j}\big\|^2+2 \mathbb{E} \big \|\v_{(n_k-1)q}-\frac{1}{N}\sum_{i=1}^{N}\nabla f(\x_{(n_k-1)q },\xi_i)\big\|^2 \nonumber\\
&+\frac{2L^2\eta^2\Delta}{d} \sum_{j=k-\tau(k)}^{k-1} \mathbb{E} \big \|  \v_{j}\big\|^2 ,
\end{align}                                                                           
where $(a)$follows from the triangle inequality, $(b)$ follows from $L$-smooth property of $f(x)$. $(c)$ is due to the triangle inequality and the maximum delay is $\Delta$.   $(d)$ uses the condition on update rule and $(e)$ follows from (\ref{eq:21}). 

Next, we continue  to bound the other term $\mathbb{E}\left \| \v_\zeta \right\|^2$ on the right-hand-side of (\ref{eq:32}). To evaluate $\mathbb{E}\left \| \v_\zeta \right\|^2$, we start from the iteration relationship of our SA-SpiderBoost algorithm.
By Assumption~\ref{assum_Lip_grad}, the entire objective function $f$ is $L-$smooth, which further implies 
\begin{align}\label{eq:26}
f(\x_{k+1})&\leq f(\x_k)+\langle\nabla f(\x_k),\x_{k+1}-\x_{k}\rangle+\frac{L}{2d}\big \| \x_{k+1}-\x_{k} \big\|^2\nonumber\\
&\overset{(a)}{\leq}  f(\x_k)+\frac{\eta}{2d} \big \| \v_k-\nabla f(\x_k) \big\|^2-(\frac{\eta}{2d}-\frac{L\eta^2}{2d})\big \| \v_k \big\|^2,
\end{align}                                                                            
where  $(a)$ uses the update rule of SA-Spiderboost,  and  the inequality that $\langle \x,\y \rangle \leq\frac{\| \x \|^2+\| \y\|^2}{2}$, for $\x,\y\in \mathbb{R} ^d$. 
Thus, we have
\begin{align} \label{eq:27}
&\mathbb{E}f(\x_{k+1})\leq \mathbb{E} f(\x_k)+\frac{\eta}{2} \mathbb{E}\big \| \v_k-\nabla f(\x_k) \big\|^2-(\frac{\eta}{2}-\frac{L\eta^2}{2})\mathbb{E}\big \| \v_k \big\|^2 \nonumber\\
\overset{(a)}{\leq}&\frac{\eta}{2d}\big (  \frac{2L^2\eta^2(\Delta+1)}{|S|d}\sum_{j=(n_k-1)q}^{k-1-\tau(k)} \mathbb{E}\big \| \v_{j}\big\|^2+2 \mathbb{E} \big \|\v_{(n_k-1)q}-\frac{1}{N}\sum_{i=1}^{N}\nabla f(\x_{(n_k-1)q },\xi_i)\big\|^2\big) \nonumber\\
&+\frac{\eta}{2d}\big ( \frac{2L^2\eta^2\Delta}{d}  \sum_{j=k-\tau(k)}^{k-1} \mathbb{E} \big \|  \v_{j}\big\|^2 \big)-	(\frac{\eta}{2d}-\frac{L\eta^2}{2d})\mathbb{E}\big \| \v_k \big\|^2+\mathbb{E} f(\x_k) \nonumber\\
\overset{(b)}{\leq}&\mathbb{E} f(\x_k)-(\frac{\eta}{2d}-\frac{L\eta^2}{2d})\mathbb{E}\big \| \v_k \big\|^2+\frac{\eta}{d}\epsilon_1^2+\frac{L^2\eta^3(\Delta+1)}{|S|d^2}\sum_{j=(n_k-1)q}^{k-1-\tau(k)} \mathbb{E}\big \| \v_{j}\big\|^2+\frac{L^2\eta^3\Delta }{d^2}\sum_{j=k-\tau(k)}^{k-1} \mathbb{E} \big \|  \v_{j}\big\|^2 \nonumber\\
\leq &\mathbb{E} f(\x_k)-(\frac{\eta}{2d}-\frac{L\eta^2}{2d})\mathbb{E}\big \| \v_k \big\|^2+\frac{\eta}{d}\epsilon_1^2+\frac{L^2\eta^3(\Delta+1)}{|S|d^2} \!\!\! \sum_{j=(n_k-1)q}^{k} \mathbb{E}\big \| \v_{j}\big\|^2+\frac{L^2\eta^3\Delta }{d^2}\sum_{j=k-\tau(k)}^{k-1} \mathbb{E} \big \|  \v_{j}\big\|^2, 
\end{align}                                                             
where $\overset{(a)}{\leq} $ follows from  (\ref{eq:22}), $\overset{(b)}{\leq} $ follows from the $\mathbb{E} \big \| \v_{(n_k-1)q}-\nabla f(\x_{(n_k-1)q}) \big\|^2 \leq \epsilon_1^2$.

Next, telescoping  (\ref{eq:27} )over $k$ from $(n_k-1)q$ to $k$ where $k\leq n_kq-1 $, we have
\begin{align}\label{eq:28}
\mathbb{E}f(\x_{k+1})\overset{(a)}{\leq} &\mathbb{E} f(\x_{(n_k-1)q})-(\frac{\eta}{2d}-\frac{L\eta^2}{2d})\sum_{i=(n_k-1)q}^k\mathbb{E}\big \| \v_i \big\|^2+\sum_{i=(n_k-1)q}^k\frac{\eta}{d}\epsilon_1^2\nonumber\\&+\frac{L^2\eta^3(\Delta+1)}{|S|d^2}\sum_{j=(n_k-1)q}^k \sum_{i=(n_k-1)q}^{j} \mathbb{E}\big \| \v_{i}\big\|^2+\frac{L^2\eta^3\Delta }{d^2} \sum_{j=(n_k-1)q}^k\sum_{i=j-\tau^j}^{j-1} \mathbb{E} \big \|  \v_{i}\big\|^2 \nonumber\\
\overset{(b)}{\leq} &\mathbb{E} f(\x_{(n_k-1)q})-(\frac{\eta}{2d}-\frac{L\eta^2}{2d})\sum_{i=(n_k-1)q}^k\mathbb{E}\big \| \v_i \big\|^2+\sum_{i=(n_k-1)q}^k\frac{\eta}{d}\epsilon_1^2\nonumber\\&+\frac{L^2\eta^3(\Delta+1)}{|S|d^2}\sum_{j=(n_k-1)q}^k \sum_{i=(n_k-1)q}^{k} \mathbb{E}\big \| \v_{i}\big\|^2+\frac{L^2\eta^3\Delta^2 }{d^2} \sum_{i=(n_k-1)q}^k  \mathbb{E} \big \|  \v_{i}\big\|^2, 
\end{align}                                                             
where $(a)$ follows from (\ref{eq:27}), $(b) $ extends the summation of second term form $j$ to $k$.
It then follows that:
\begin{align} \label{eq:29}
\mathbb{E}f(\x_{k+1})
\overset{(a)}{\leq} &\mathbb{E} f(\x_{(n_k-1)q})-(\frac{\eta}{2d}-\frac{L\eta^2}{2d})\sum_{i=(n_k-1)q}^k\mathbb{E}\big \| \v_i \big\|^2+\sum_{i=(n_k-1)q}^k\frac{\eta}{d}\epsilon_1^2 \nonumber\\
&+\frac{qL^2\eta^3(\Delta+1)}{|S|d^2}\sum_{i=(n_k-1)q}^k \mathbb{E}\big \| \v_{i}\big\|^2+\frac{L^2\eta^3\Delta^2 }{d^2}\sum_{j=(n_k-1)q}^k \mathbb{E} \big \|  \v_{i}\big\|^2  \nonumber\\
=&\mathbb{E} f(\x_{(n_k-1)q})+\sum_{i=(n_k-1)q}^k \frac{\eta}{d}\epsilon_1^2-\big[ \frac{\eta}{2d}-\frac{L\eta^2}{2d}- \frac{L^2\eta^3}{d^2}\big(\frac{q(\Delta+1)}{|S|}+\Delta^2 \big) \big] \sum_{i=(n_k-1)q}^k\mathbb{E} \big \| \v_{i} \big\|^2  \nonumber\\
=&\mathbb{E}[f(\x_{(n_k-1)q})]-\sum_{i=(n_k-1)q}^{k}(\beta_1\mathbb{E}\big \| \v_i \big\|^2-\frac{\eta}{d}\epsilon_1^2),
\end{align}
where $(a)$ follows from (\ref{eq:28}).
Then, we can further derive: 
\begin{align}\label{eq:30}
\mathbb{E}&f(\x_K)-\mathbb{E}f(\x_0)=(\mathbb{E}f(\x_q)-\mathbb{E}f(\x_0))+(\mathbb{E}f(\x_{2q})-\mathbb{E}f(\x_q)+...+(\mathbb{E}f(\x_{K})-\mathbb{E}f(\x_{(n_k-1)q}))) \nonumber\\
\overset{(a)}{\leq}&-\sum_{i=0}^{q-1}(\beta_1\mathbb{E}\big \| \v_i \big\|^2-\frac{\eta}{d}\epsilon_1^2)-\sum_{i=q}^{2q-1}(\beta_1\mathbb{E}\big \| \v_i \big\|^2-\frac{\eta}{d}\epsilon_1^2)-... -\sum_{(n_K-1)q}^{K-1}(\beta_1\mathbb{E}\big \| \v_i \big\|^2-\frac{\eta}{d}\epsilon_1^2) \nonumber\\
=&-\sum_{i=0}^{K-1}(\beta_1\mathbb{E}\big \| \v_i \big\|^2-\frac{\eta}{d}\epsilon_1^2) =
 -\sum_{i=0}^{K-1}(\beta_1\mathbb{E}\big \| \v_i \big\|^2)+\frac{K\eta}{d}\epsilon_1^2,
\end{align}        
where $(a)$ follows from eq.(\ref{eq:29}),. Since $\mathbb{E}f(\x_K)\geq f(x^*)$, then we have
\begin{equation}\label{eq:31}
\begin{aligned}
&\mathbb{E}f(x^*)-\mathbb{E}f(\x_0)\leq  -\sum_{i=0}^{K-1}(\beta_1\mathbb{E}\big \| \v_i \big\|^2)+\frac{K\eta}{d}\epsilon_1^2.
\end{aligned}                                                             
\end{equation}   
By rearranging (\ref{eq:31}), we have:
\begin{equation}\label{eq:33}
\begin{aligned}
\mathbb{E}\big \| \v_\zeta \big\|^2=\frac{1}{K}\sum_{i=0}^{K-1}\mathbb{E}\big \| \v_i\big\|^2\leq \frac{f(\x_0)-f^*}{K\beta_1}+\frac{\eta}{d\beta_1}\epsilon_1^2.
\end{aligned}                                                             
\end{equation}   
It then follow that:
\begin{align}\label{eq:34}
\mathbb{E} \big \| \v_\zeta-\nabla f(\x_\zeta) \big\|^2&
\overset{(a)}{\leq}\mathbb{E}\frac{2L^2\eta^2(\Delta+1)}{|S|d}\sum_{j=(n_{\zeta}-1)q}^{\zeta-1-\tau^{\zeta}} \mathbb{E}\big \| \v_{j}\big\|^2+2 \mathbb{E} \big \|\v_{(n_k-1)q}-\frac{1}{N}\sum_{i=1}^{N}\nabla f(\x_{(n_k-1)q },\xi_i)\big\|^2\nonumber\\&+\frac{2L^2\eta^2\Delta}{d} \mathbb{E}\sum_{j=\zeta-\tau^{\zeta}}^{\zeta-1} \mathbb{E} \big \|  \v_{j}\big\|^2 \nonumber\\
&\overset{(b)}{\leq}\frac{2L^2\eta^2(\Delta+1)}{|S|d}\mathbb{E}\sum_{j=(n_{\zeta}-1)q}^{\zeta-1-\tau^{\zeta}} \mathbb{E}\big \| \v_{j}\big\|^2+2 \epsilon_1^2+\frac{2L^2\eta^2\Delta}{d} \mathbb{E} \sum_{i=\zeta-\tau^{\zeta}}^{\zeta-1} \mathbb{E} \big \|  \v_{i}\big\|^2\nonumber \\
&\leq\frac{2L^2\eta^2(\Delta+1)}{|S|d}\mathbb{E}\sum_{j=(n_{\zeta}-1)q}^{\xi} \mathbb{E}\big \| \v_{j}\big\|^2+2 \epsilon_1^2+\frac{2L^2\eta^2\Delta}{d}  \mathbb{E}\sum_{i=\zeta-\tau^{\zeta}}^{\zeta-1} \mathbb{E} \big \|  \v_{i}\big\|^2 \nonumber\\
&\overset{(c)}{\leq}\frac{2L^2\eta^2(\Delta+1)}{|S|d}\mathbb{E}\sum_{i=(n_{\zeta}-1)q}^{min\{(n_{\zeta})q-1,K-1\}} \mathbb{E}\big \| \v_{i}\big\|^2+2 \epsilon_1^2+\frac{2L^2\eta^2\Delta}{d}  \mathbb{E}\sum_{i=\zeta-\tau^{\zeta}}^{\zeta-1} \mathbb{E} \big \|  \v_{i}\big\|^2 \nonumber\\
&\overset{(d)}{\leq}\frac{q}{K}\sum_{i=0}^{K-1}\frac{2L^2\eta^2(\Delta+1)}{|S|d} \mathbb{E}\big \| \v_{i}\big\|^2+2 \epsilon_1^2+\frac{\Delta}{K}\sum_{i=0}^{K-1}\frac{2L^2\eta^2\Delta}{d}  \mathbb{E} \big \|  \v_{i}\big\|^2 \nonumber\\
&=(\frac{2L^2\eta^2q(\Delta+1)}{|S|d}+\frac{2L^2\eta^2\Delta^2}{d})\frac{1}{K} \sum_{i=0}^{K-1}\mathbb{E}\big \| \v_{i}\big\|^2+2 \epsilon_1^2 \nonumber\\
&\overset{(e)}{\leq}\big[\frac{{2L^2}(\Delta^2+\frac{q(\Delta+1)}{|S|})\eta^3}{\beta_1d^2}+2\big]\epsilon_1^2+ \frac{{2L^2}(\Delta^2+\frac{q(\Delta+1)}{|S|})\eta^2}{d}\big[\frac{f(\x_0)-f^*}{K\beta_1}\big],
\end{align}               
where $(a)$ follows from (\ref{eq:25}), $(b)$ follows from (\ref{eq:27}), 
$(c)$ follows
from the definition of $n_{\zeta}$, which implies $\xi \leq min\{(n_{\zeta})q-1,K-1\}$, 
$(d)$follows from the fact that the probability that $n_{\zeta}=1, 2,...,n_K$ is less than or equal to $\frac{q}{K}$,
and $(e)$follows from (\ref{eq:35}).
Then we can obtain

\begin{align}\label{eq:35}
\mathbb{E}&\big \| \nabla f(\x_\zeta) \big\|^2\leq 2\mathbb{E}\big \| \nabla f(\x_\zeta)-\v_\zeta\big\|^2+2\mathbb{E}\big \| \v_\zeta \big\|^2\nonumber\\
\overset{(a)}{\leq}& 2\big\{\big[\frac{{2L^2}(\Delta^2+\frac{q(\Delta+1)}{|S|})\eta^3}{\beta_1d^2}+2\big]\epsilon_1^2+\frac{ {2L^2}(\Delta^2+\frac{q(\Delta+1)}{|S|})\eta^2}{d}\big[\frac{f(\x_0)-f^*}{K\beta_1}\big]\big\}\nonumber\\&+ 2[\frac{f(\x_0)-f^*}{K\beta_1}+\frac{\eta}{\beta_1d}\epsilon_1^2]\nonumber\\
=&\big[\frac{2}{\beta_1d}\big(\eta+\frac{2L^2}{d}(\Delta^2+\frac{q(\Delta+1)}{|S|})\eta^3 \big)+4\big]\epsilon_1^2+\big[ \frac{{4L^2}(\Delta^2+\frac{q(\Delta+1)}{|S|})\eta^2}{K\beta_1d}+\frac{2}{K\beta_1} \big]  ({f(\x_0)-f^*}),
\end{align}        
where $(a)$ follows the (\ref{eq:34}) and (\ref{eq:35}).
This completes the proof.
\end{proof}

To wrap up the proof of the theorem, we set the parameters as:
\begin{equation}
\begin{aligned}
q=\sqrt{N}, S=\sqrt{N}, \quad \eta=\frac{1}{2L(\Delta+1)}. 
\end{aligned}                                                             
\end{equation}  
Then, we obtain: 
\begin{equation}
\begin{aligned}
\beta_1&= \frac{\eta}{2d}-\frac{L\eta^2}{2d}- \frac{L^2\eta^3}{d^2}\big(\frac{q(\Delta+1)}{|S|}+\Delta^2 \big) 
=\frac{1}{4Ld(\Delta+1)}\cdot\big(1-\frac{1}{2(\Delta+1)}-\frac{(1+\Delta+\Delta^2)}{2d(\Delta+1)^2}\big)>0
\end{aligned}                                                             
\end{equation}  
for $mod(k,q)=0$ we have $\mathbb{E} \big \| \v_k-\nabla f(\x_k) \big\|^2=0$. Then, after $K$ iterations, we have

\begin{equation}\label{eq:18}
\begin{aligned}
\mathbb{E}\big \| \nabla f(\x_\zeta) \big\|^2\leq \frac{8Ld(\Delta+1)}{K}\frac{2(\Delta+1)^2d+\Delta^2+\Delta+1}{2d(\Delta+1)^2-d(\Delta+1)-(\Delta^2+\Delta+1)} (f(\x_0)-f^*).
\end{aligned}                                                             
\end{equation}  
To ensure $\mathbb{E}\big \| \nabla f(\x_\zeta) \big\|\leq\epsilon$, it suffices to ensure$\mathbb{E}\big \| \nabla f(\x_\zeta) \big\|^2\leq\epsilon^2$.
Solving for $K$ yields:
	\begin{align}
	\begin{split}
		K = \frac{8Ld(\Delta+1)}{\epsilon^2}\frac{2(\Delta+1)^2d+\Delta^2+\Delta+1}{2d(\Delta+1)^2-d(\Delta+1)-(\Delta^2+\Delta+1)}  (f(\x_0)-f^*) = \mathcal{O} \big(  \frac{ f(\x_0)-f^* }{\epsilon^2} \big).
	\end{split}
\end{align}

Similar to Theorem \ref{aaa}, the total SFO complexity can be calculated as:$\lceil \frac{K}{q} \rceil N+ K\cdot S\leq  \frac{K+q}{q}N+ K\cdot S= K\sqrt{N}+N+K\sqrt{N}=O(\sqrt{N}\epsilon^{-2}(\Delta+1)  d+N)$.
This completes the proof.

\section{ Generalization analysis of SA-SpiderBoost for shared memory }  \label{generalization_shared}
\subsection{Proofs of Theorem~\ref{generalization_shared_Theorem} }  
\begin{proof}
Recall that update rule for SA-SpiderBoost for shared-memory system is given by:
\begin{align}
(\x_{k+1})_{m_k}=(\x_{k})_{m_k}-\eta\frac{1}{|S|} \sum_{i\in S} (\nabla f(\x_{k-\tau(k)},\xi_i)-\nabla f(\x_{k-\tau(k)-1-\tau(k-\tau(k)-1)},\xi_i)+\v_{k-1-\tau(k)}) _{m_k}.
\end{align}

Let $S$ and $S'$ be two data sets such that $S$ and $S'$ differ in at most one example, where $S=(\xi_1,\xi_2,...,\xi_N)$ and $S'=(\xi_1',\xi_2',...,\xi_N')$.  Let $\delta_k\triangleq \x_k-\x_k' $. Suppose $\x_0=\x_0'$.

Now, taking expectation of $\delta_{k+1}$ with respect of the algorithm, we get 
\begin{align}
\begin{split}
\mathbb{E} (\delta&_{k+1})_{m_k}=\mathbb{E}  (\x_{k+1})_{m_k}-\mathbb{E} (\x_{k+1}')_{m_k}
=\mathbb{E} (\delta_k)_{m_k}-\eta [\mathbb{E} (\v_k)_{m_k}-\mathbb{E} (\v_k')_{m_k}].
\end{split}
\end{align}
Since we $\v_k$ is the unbiased estimated of $\nabla f(\x_{k-\tau(k)})$, we have:
\begin{align}
\begin{split}
\mathbb{E} (\delta&_{k+1})_{m_k}
=\mathbb{E} (\delta_k)_{m_k}-\eta [\mathbb{E} (\nabla f(\x_{k-\tau(k)}))_{m_k}-\mathbb{E} (\nabla f(\x_{k-\tau(k)}'))_{m_k}].
\end{split}
\end{align}
At Step $k$, with probability $1-1/N$, the example is the same in  $S$ and $S'$. With probability $\frac{1}{N}$, the example is different in  $S$ and $S'$. Hence, we have
\begin{align}
\mathbb{E}  (\delta_{k+1})_{m_k}
\leq& \mathbb{E} \big[(\delta_{k})_{m_k}-\eta  \frac{1}{N}\sum_{i=1}^N\nabla f((\x_{k-\tau(k)},\xi_i))_{m_k}+\eta \frac{1}{N}\sum_{i=1}^N\nabla f((\x_{k-\tau(k)}',\xi_i))_{m_k}\big]\nonumber \\&+\mathbb{E} \big[\frac{1}{N}\eta\nabla f((\x_{k-\tau(k)}',\xi_i))_{m_k}  -\frac{1}{N}\eta\nabla f((\x_{k-\tau(k)}',\xi_i'))_{m_k} \big] \nonumber\\
\overset{(a)}{\leq}& \mathbb{E}\big[(\delta_{k})_{m_k}-\eta  A (\delta_{k-\tau(k)})_{m_{k}}\big]+\epsilon'',
\end{align}
where  $(a)$ follows from $f$ is a quadratic convex function. We note that $\|\epsilon''\|\leq \frac{2\eta M}{N\sqrt{d}}$ since we have the bounded gradient Assumption \ref{assum_bnded_grad} and  $m_k\in{1,2,...,d}$ is the uniformly selected updated coordinate in $x$ in iteration $k$.
Then, we have

\begin{align}
\left[
\begin{matrix}
(\delta_{k+1})_{m_k} \\
(\delta_{k} )_{m_k}\\
\vdots\\
\delta_{k-\tau(k)+1} \\
\delta_{k-\tau(k)} \\
\end{matrix}
\right]=
\left[
	\begin{matrix}
		1 & 0 & \cdots &- \eta A & 0\\
		1&0& \cdots &0&0\\
		\vdots & \vdots & \ddots & \vdots & \vdots \\ 
		\vdots & \vdots & & \ddots & \vdots \\
		0&0& \cdots &1&0\\
	\end{matrix}
\right]
\left[
\begin{matrix}
(\delta_{k} )_{m_k} \\
(\delta_{k-1} )_{m_k} \\
\vdots\\
(\delta_{k-\tau(k)} )_{m_k} \\
(\delta_{k-\tau(k)-1} )_{m_k} \\
\end{matrix}
\right]+
\left[
\begin{matrix}
\epsilon''\\
0 \\\vdots\\0\\0
\end{matrix}
\right].
\end{align}

Let Matrix 
\begin{align}
\Q\triangleq
\left[
	\begin{matrix}
		1 & 0 & \cdots &- \eta A & 0\\
		1&0& \cdots &0&0\\
		\vdots & \vdots & \ddots & \vdots & \vdots \\ 
		\vdots & \vdots & & \ddots & \vdots \\
		0&0& \cdots &1&0\\
	\end{matrix}
\right].
\end{align}
Consider the characteristic polynomial 

\begin{align}
\left[
	\begin{matrix}
		1 & 0 & \cdots &- \eta A & 0\\
		1&0& \cdots &0&0\\
		\vdots & \vdots & \ddots & \vdots & \vdots \\ 
		\vdots & \vdots & & \ddots & \vdots \\
		0&0& \cdots &1&0\\
	\end{matrix}
\right]
\left[
\begin{matrix}
\v_1\\
\v_2\\
\vdots\\ \vdots \\
\v_{\tau(k)+2}
\end{matrix}
\right]
=\lambda_\Q
\left[
\begin{matrix}
\v_1\\
\v_2\\
\vdots \\ \vdots \\
\v_{\tau(k)+2}
\end{matrix}
\right],
\end{align}
which implies $\v_1=\lambda_\Q \v_2$., $\v_2=\lambda_\Q \v_3$, ...,$\v_{\tau^{k}+1}=\lambda_\Q \v_{\tau(k)+2}$ plugging this in to the first row, then we have

\begin{align}
\big(-\lambda_\Q^{\tau^{k}+2} +\lambda_\Q^{\tau^{k}+1} + \lambda_\Q[-\eta A]  \big) \v_{\tau(k)+2}=0.
\end{align}

Since A is symmetric, then it has eigenvalue decomposition $\A=\mathbf{U} \boldsymbol{\Lambda} \mathbf{U}^{\top}$.
Then equation can be written as:
\begin{align}
\mathbf{U} \big(-\lambda_\Q^{\tau^{k}+2} +\lambda_\Q^{\tau^{k}+1} + \lambda_\Q[-\eta A]  \big) \mathbf{U}^{\top}=0.
\end{align}
Let $\lambda_1,...,\lambda_ b\in[\mu,M]$ be eigenvalue of symmetric matrix A. Then we have
\begin{align}
\lambda_\Q^{\tau_k+1}\cdot \big[-\lambda_\Q+1\big]=\lambda_\Q [\eta\lambda_i].
\end{align}
Since $0< \eta<\frac{1}{ M},d \geq 1, \lambda_1,...,\lambda_b \in[\mu,M]$, then we have $0\leq \eta \lambda_i\leq1$. Thus, $max(|\lambda_\Q|)=1$.
 Thus, for $mod(k,q)\neq0$ we have
\begin{align} \label{eq:36}
\mathbb{E}	\|(\delta_{k+1})_{m_k}\|\leq \mathbb{E}\|(\delta_{k})_{m_k}\|+\frac{2\eta M}{N\sqrt{d}}. 
\end{align}

For those $k$-values that satisfy $mod(k,q)=0$, we can also show $\|(\delta_{k+1})_{m_k}\|=\|(\x_{k+1})_{m_k}-(\x_{k+1}')_{m_k}\|\leq\|(\delta_{k})_{m_k}\|+\frac{2\eta M}{N\sqrt{d}}$. 
Also, from (\ref{eq:36}), we always have $\|(\delta_{k+1})_{m_k}\|\leq \|(\delta_{k})_{m_k}\|+\frac{2\eta M}{N\sqrt{d}}$. 
By applying this bound inductively, we can bound $\delta_{k+1}$ using the total number $K$  iterations as:
\begin{align}
\mathbb{E}\parallel \delta_{k+1}\parallel\leq \frac{2\eta MK}{N\sqrt{d}}.
\end{align}
Similar to Theorem \ref{quadratic_generalization_distributed}, the  SA-SpiderBoost algorithm has the following stability bound:

	\begin{align}
	\begin{split}
		\epsilon'\leq M\cdot\mathbb{E}\parallel \delta_{t+1}\parallel \leq \frac{2\eta M^2K}{N\sqrt{d}}.
	\end{split}
\end{align}
This completes the proof.
\end{proof}

\subsection{Proofs of Theorem~\ref{generalization_shared_Theorem} } \label{appdx:nonconvex_gen_shared}

\begin{proof}
Recall that update rule for SpiderBoost is given by:
\begin{align}
(\x_{k+1})_{m_k}=(\x_{k})_{m_k}-\eta\frac{1}{|S|} \sum_{i\in S} (\nabla f(\x_{k-\tau(k)},\xi_i)-\nabla f(\x_{k-\tau(k)-1-\tau(k-\tau(k)-1)},\xi_i)+\v_{k-1-\tau(k)}) _{m_k}.
\end{align}
Let $S$ and $S'$ be two data sets such that $S$ and $S'$ differ in at most one example, where $S=(\xi_1,\xi_2,...,\xi_N)$ and $S'=(\xi_1',\xi_2',...,\xi_N')$. 
Let $\delta_k\triangleq \x_k-\x_k' $. Suppose $\x_0=\x_0'$.

Now, taking expectation of $\delta_{k+1}$ with respect of the algorithm, we get 
\begin{align}
\begin{split}
\mathbb{E} (\delta&_{k+1})_{m_k}=\mathbb{E}  (\x_{k+1})_{m_k}-\mathbb{E} (\x_{k+1}')_{m_k}
=\mathbb{E} (\delta_k)_{m_k}-\eta [\mathbb{E} (\v_k)_{m_k}-\mathbb{E} (\v_k')_{m_k}].
\end{split}
\end{align}
Since we $\v_k$ is the unbiased estimated of $\nabla f(\x_{k-\tau(k)})$. we have
\begin{align}
\begin{split}
\mathbb{E} (\delta&_{k+1})_{m_k}
=\mathbb{E} (\delta_k)_{m_k}-\eta [\mathbb{E} (\nabla f(\x_{k-\tau(k)}))_{m_k}-\mathbb{E} (\nabla f(\x_{k-\tau(k)}'))_{m_k}].\\
\end{split}
\end{align}
At Step $k$, with probability $1-1/N$, the example is the same in  $S$ and $S'$. With probability $\frac{1}{N}$, the example is different in  $S$ and $S'$. Hence, we have
\begin{align}
\mathbb{E}\| (\delta_{k+1})_{m_k}\|
\leq& \mathbb{E}\| (\delta_{k})_{m_k}-\eta  (\frac{1}{N}\sum_{i=1}^N\nabla f(\x_{k-\tau(k)},\xi_i))_{m_k}+\eta (\frac{1}{N}\sum_{i=1}^N\nabla f(\x_{k-\tau(k)}',\xi_i))_{m_k}\|\nonumber\\&+\frac{1}{N}\eta\| \nabla f(\x_{k-\tau(k)}',\xi_i))_{m_k}  -\nabla f(\x_{k-\tau(k)}',\xi_i'))_{m_k}  \|\nonumber\\
\overset{(a)}{\leq}& \mathbb{E}\| (\delta_{k})_{m_k}-\eta  (\frac{1}{N}\sum_{i=1}^N\nabla f(\x_{k-\tau(k)},\xi_i))_{m_k}+\eta (\frac{1}{N}\sum_{i=1}^N\nabla f(\x_{k-\tau(k)}',\xi_i))_{m_k}\|+\frac{2\eta M}{N\sqrt{d}}\nonumber\\
\overset{(b)}{\leq}& \mathbb{E}\| (\delta_{k})_{m_k}\|+\eta  \mathbb{E}\| (\frac{1}{N}\sum_{i=1}^N\nabla f(\x_{k-\tau(k)},\xi_i))_{m_k}- (\frac{1}{N}\sum_{i=1}^N\nabla f(\x_{k-\tau(k)}',\xi_i))_{m_k}\|+\frac{2\eta M}{N\sqrt{d}}\nonumber\\
\overset{(c)}{\leq}& \mathbb{E}\| (\delta_{k})_{m_k}\|+\frac{2 \eta M}{\sqrt{d}}+\frac{2\eta M}{N\sqrt{d}},
\end{align}
where $(a)$ and $(b)$ follows from the triangle inequality, $(c)$ follows from the bounded gradient Assumption \ref{assum_bnded_grad}. We note that $\eta\| \nabla f(\x_{k-\tau(k)}',\xi_i))_{m_k}  -\nabla f(\x_{k-\tau(k)}',\xi_i'))_{m_k}  \|\leq \frac{2\eta M}{\sqrt{d}}$ derives form the bounded gradient Assumption \ref{assum_bnded_grad} and  $m_k\in{1,2,...,d}$ is the uniformly selected updated coordinate in $x$ in iteration $k$.

While $i\neq m_k$ , we have $ (\delta_{k})_{i+1} =(\x_{k})_{i+1}-(\x_{k}')_{i+1}==(\x_{k})_{i}-(\x_{k}')_{i}=(\delta_{k})_{i} $  Then we have 
\begin{align}
\begin{split}
\mathbb{E}\| (\delta_{k+1})\|\leq & \mathbb{E}\| (\delta_{k})\|+\frac{2 \eta M}{\sqrt{d}}+\frac{2\eta M}{N\sqrt{d}}. \\
\end{split}
\end{align}

Thus,  total number K of iterations satisfies
\begin{align}
\parallel \delta_{k+1}\parallel\leq \frac{2 \eta M K}{\sqrt{d}}+\frac{2\eta MK}{N\sqrt{d}}.
\end{align}

Similar to Theorem \ref{nonconvex_gen_dis}, the  SA-SpiderBoost algorithm has the following stability bound:
\begin{align}
\begin{split}
\epsilon'\leq M\cdot\parallel \delta_{t+1}\parallel \leq  \frac{2 \eta M^2 K}{\sqrt{d}}+\frac{2\eta M^2K}{N\sqrt{d}}.
\end{split}
\end{align}
This completes the proof.
\end{proof}

\end{document}

%% file: symbols_commands.tex
\newcommand{\A}{\mathbf{A}}

\renewcommand{\a}{\mathbf{a}}

\newcommand{\Q}{\mathbf{Q}}

\renewcommand{\u}{\mathbf{u}}

\renewcommand{\v}{\mathbf{v}}

\newcommand{\x}{\mathbf{x}}

\newcommand{\y}{\mathbf{y}}

\newcommand{\0}{\mathbf{0}}

\newcommand{\norm}[1]{\mbox{$\left\lVert #1 \right\rVert$}}

\newtheorem{defn}{Definition}

\newtheorem{assum}{Assumption}



%% file: main.bbl
\begin{thebibliography}{10}

\bibitem{abadi2016tensorflow}
{\sc Abadi, M., Agarwal, A., Barham, P., Brevdo, E., Chen, Z., Citro, C.,
  Corrado, G.~S., Davis, A., Dean, J., Devin, M., et~al.}
\newblock Tensorflow: Large-scale machine learning on heterogeneous distributed
  systems.
\newblock {\em arXiv preprint arXiv:1603.04467\/} (2016).

\bibitem{agarwal2011distributed}
{\sc Agarwal, A., and Duchi, J.~C.}
\newblock Distributed delayed stochastic optimization.
\newblock In {\em Advances in Neural Information Processing Systems\/} (2011),
  pp.~873--881.

\bibitem{Hassibi19:ICLR}
{\sc Azizan, N., Lale, S., and Hassibi, B.}
\newblock Stochastic mirror descent on overparameterized nonlinear models:
  Convergence, implicit regularization, and generalization.
\newblock {\em arXiv:1906.03830\/} (2019).

\bibitem{bousquet2002stability}
{\sc Bousquet, O., and Elisseeff, A.}
\newblock Stability and generalization.
\newblock {\em Journal of machine learning research 2}, Mar (2002), 499--526.

\bibitem{defazio2014saga}
{\sc Defazio, A., Bach, F., and Lacoste-Julien, S.}
\newblock Saga: A fast incremental gradient method with support for
  non-strongly convex composite objectives.
\newblock In {\em Advances in neural information processing systems\/} (2014),
  pp.~1646--1654.

\bibitem{fang2018spider}
{\sc Fang, C., Li, C.~J., Lin, Z., and Zhang, T.}
\newblock Spider: Near-optimal non-convex optimization via stochastic
  path-integrated differential estimator.
\newblock In {\em Advances in Neural Information Processing Systems\/} (2018),
  pp.~689--699.

\bibitem{george2006adaptive}
{\sc George, A.~P., and Powell, W.~B.}
\newblock Adaptive stepsizes for recursive estimation with applications in
  approximate dynamic programming.
\newblock {\em Machine learning 65}, 1 (2006), 167--198.

\bibitem{ghadimi2013stochastic}
{\sc Ghadimi, S., and Lan, G.}
\newblock Stochastic first-and zeroth-order methods for nonconvex stochastic
  programming.
\newblock {\em SIAM Journal on Optimization 23}, 4 (2013), 2341--2368.

\bibitem{hardt2016train}
{\sc Hardt, M., Recht, B., and Singer, Y.}
\newblock Train faster, generalize better: Stability of stochastic gradient
  descent.
\newblock In {\em International Conference on Machine Learning\/} (2016),
  pp.~1225--1234.

\bibitem{huo2016asynchronous}
{\sc Huo, Z., and Huang, H.}
\newblock Asynchronous stochastic gradient descent with variance reduction for
  non-convex optimization.
\newblock {\em arXiv preprint arXiv:1604.03584\/} (2016).

\bibitem{johnson2013svrg}
{\sc Johnson, R., and Zhang, T.}
\newblock Accelerating stochastic gradient descent using predictive variance
  reduction.
\newblock In {\em Advances in neural information processing systems\/} (2013),
  pp.~315--323.

\bibitem{johnson2013accelerating}
{\sc Johnson, R., and Zhang, T.}
\newblock Accelerating stochastic gradient descent using predictive variance
  reduction.
\newblock In {\em Advances in neural information processing systems\/} (2013),
  pp.~315--323.

\bibitem{lecun2010mnist}
{\sc LeCun, Y., Cortes, C., and Burges, C.}
\newblock Mnist handwritten digit database.
\newblock {\em Available: http://yann. lecun. com/exdb/mnist\/} (1998).

\bibitem{li2021page}
{\sc Li, Z., Bao, H., Zhang, X., and Richtarik, P.}
\newblock Page: A simple and optimal probabilistic gradient estimator for
  nonconvex optimization.
\newblock In {\em International Conference on Machine Learning\/} (2021), PMLR,
  pp.~6286--6295.

\bibitem{lian2015asynchronous}
{\sc Lian, X., Huang, Y., Li, Y., and Liu, J.}
\newblock Asynchronous parallel stochastic gradient for nonconvex optimization.
\newblock In {\em Advances in Neural Information Processing Systems\/} (2015),
  pp.~2737--2745.

\bibitem{proof}
{\sc Liu, Z., Zhang, X., and Liu, J.}
\newblock Synthesis: A semi-asynchronous path-integrated stochastic gradient
  method for distributed learning in computing clusters.
\newblock \url{https://kevinliu-osu.github.io/publications/SYNTHESIS_TR.pdf}.

\bibitem{nguyen2017sarah}
{\sc Nguyen, L.~M., Liu, J., Scheinberg, K., and Takavc, M.}
\newblock Sarah: A novel method for machine learning problems using stochastic
  recursive gradient.
\newblock In {\em International Conference on Machine Learning\/} (2017), PMLR,
  pp.~2613--2621.

\bibitem{peng2016arock}
{\sc Peng, Z., Xu, Y., Yan, M., and Yin, W.}
\newblock Arock: an algorithmic framework for asynchronous parallel coordinate
  updates.
\newblock {\em SIAM Journal on Scientific Computing 38}, 5 (2016),
  A2851--A2879.

\bibitem{polyak1964Momentum}
{\sc Polyak, B.~T.}
\newblock Some methods of speeding up the convergence of iteration methods.
\newblock {\em USSR Computational Mathematics and Mathematical Physics 4}, 5
  (1964), 1--17.

\bibitem{recht2011hogwild}
{\sc Recht, B., Re, C., Wright, S., and Niu, F.}
\newblock Hogwild: A lock-free approach to parallelizing stochastic gradient
  descent.
\newblock In {\em Advances in neural information processing systems\/} (2011),
  pp.~693--701.

\bibitem{reddi2015variance}
{\sc Reddi, S.~J., Hefny, A., Sra, S., Poczos, B., and Smola, A.~J.}
\newblock On variance reduction in stochastic gradient descent and its
  asynchronous variants.
\newblock In {\em Advances in Neural Information Processing Systems\/} (2015),
  pp.~2647--2655.

\bibitem{robbins1951stochastic}
{\sc Robbins, H., and Monro, S.}
\newblock A stochastic approximation method.
\newblock {\em The annals of mathematical statistics\/} (1951), 400--407.

\bibitem{roux2012stochastic}
{\sc Roux, N.~L., Schmidt, M., and Bach, F.~R.}
\newblock A stochastic gradient method with an exponential convergence \_rate
  for finite training sets.
\newblock In {\em Advances in neural information processing systems\/} (2012),
  pp.~2663--2671.

\bibitem{schmidt2017SAG}
{\sc Schmidt, M., Le~Roux, N., and Bach, F.}
\newblock Minimizing finite sums with the stochastic average gradient.
\newblock {\em Mathematical Programming 162}, 1-2 (2017), 83--112.

\bibitem{tang2015providing}
{\sc Tang, L., Harrington, P., and Zhu, T.}
\newblock Providing personalized item recommendations using scalable matrix
  factorization with randomness, Feb.~19 2015.
\newblock US Patent App. 13/970,271.

\bibitem{tavara2019svm}
{\sc Tavara, S.}
\newblock Parallel computing of support vector machines: a survey.
\newblock {\em ACM Computing Surveys (CSUR) 51}, 6 (2019), 1--38.

\bibitem{wang2018spiderboost}
{\sc Wang, Z., Ji, K., Zhou, Y., Liang, Y., and Tarokh, V.}
\newblock Spiderboost: A class of faster variance-reduced algorithms for
  nonconvex optimization.
\newblock {\em arXiv preprint arXiv:1810.10690\/} (2018).

\bibitem{wen2017terngrad}
{\sc Wen, W., Xu, C., Yan, F., Wu, C., Wang, Y., Chen, Y., and Li, H.}
\newblock Terngrad: Ternary gradients to reduce communication in distributed
  deep learning.
\newblock In {\em Advances in neural information processing systems\/} (2017),
  pp.~1509--1519.

\bibitem{yao2018scalable}
{\sc Yao, Q., and Kwok, J.}
\newblock Scalable robust matrix factorization with nonconvex loss.
\newblock In {\em Advances in Neural Information Processing Systems\/} (2018),
  pp.~5061--5070.

\bibitem{zhang2014asynchronous}
{\sc Zhang, R., and Kwok, J.}
\newblock Asynchronous distributed admm for consensus optimization.
\newblock In {\em International conference on machine learning\/} (2014),
  pp.~1701--1709.

\bibitem{zhang2018taming}
{\sc Zhang, X., Liu, J., and Zhu, Z.}
\newblock Taming convergence for asynchronous stochastic gradient descent with
  unbounded delay in non-convex learning.
\newblock In {\em Proc. IEEE CDC\/} (2020).

\bibitem{zhao2016fast}
{\sc Zhao, S.-Y., and Li, W.-J.}
\newblock Fast asynchronous parallel stochastic gradient descent: A lock-free
  approach with convergence guarantee.
\newblock In {\em Thirtieth AAAI conference on artificial intelligence\/}
  (2016).

\end{thebibliography}
